\documentclass[11pt]{article}

\usepackage[numbers, compress]{natbib}

\usepackage[final]{neurips_2024}

\usepackage[utf8]{inputenc} %
\usepackage[T1]{fontenc}    %
\usepackage{hyperref}       %
\usepackage{url}            %
\usepackage{booktabs}       %
\usepackage{amsfonts}       %
\usepackage{nicefrac}       %
\usepackage{microtype}      %
\usepackage{xcolor}         %
\usepackage{subfiles}
\usepackage{mathtools, amsthm, amssymb}
\usepackage{thmtools}

\DeclareMathOperator{\Tr}{Tr}

\newcommand{\be}{\boldsymbol{e}}
\newcommand{\bg}{\boldsymbol{g}}
\newcommand{\bh}{\boldsymbol{h}}

\newcommand{\bx}{\boldsymbol{x}}
\newcommand{\bu}{\boldsymbol{u}}

\newcommand{\bA}{\boldsymbol{A}}
\newcommand{\bB}{\boldsymbol{B}}
\newcommand{\bC}{\boldsymbol{C}}
\newcommand{\bH}{\boldsymbol{H}}
\newcommand{\bI}{\boldsymbol{I}}
\newcommand{\bM}{\boldsymbol{M}}
\newcommand{\bS}{\boldsymbol{S}}

\newcommand{\bY}{\boldsymbol{Y}}
\newcommand{\bD}{\boldsymbol{D}}

\newcommand{\bw}{\boldsymbol{w}}

\newcommand{\bv}{\boldsymbol{v}}

\newcommand{\bSigma}{\boldsymbol{\Sigma}}
\newcommand{\btheta}{\boldsymbol{\theta}}

\newcommand{\argmin}{\mathop{\mathrm{argmin}}}

\newcommand{\field}[1]{\mathbb{#1}}

\newcommand{\R}{\field{R}}

\newcommand{\Var}{\mathrm{Var}}

\newcommand{\norm}[1]{\left\|{#1}\right\|}

\newcommand{\scO}{\mathcal{O}}

\DeclareMathOperator*{\Exp}{\mathbf{E}}

\DeclareMathOperator{\Wealth}{Wealth}

\newcommand{\indicator}{\mathbf{1}}

\renewcommand{\tilde}{\widetilde}

\newcommand{\inner}[1]{\left\langle#1\right\rangle}
\newcommand{\tmm}{{t-1}}
\newcommand{\tpp}{{t+1}}
\newcommand{\half}{\frac{1}{2}}
\newcommand{\zeros}{\mathbf{0}}
\newcommand{\cmp}{\bu}

\newcommand{\inv}{{-1}}
\newcommand{\sbrac}[1]{\left[#1\right]}
\newcommand{\brac}[1]{\left(#1\right)}
\newcommand{\grad}{\nabla}
\newcommand{\sumtT}{\sum_{t=1}^{T}}
\newcommand{\sumtTmm}{\sum_{t=1}^{T-1}}
\newcommand{\Log}[1]{\log\left(#1\right)}
\newcommand{\w}{\bw}
\newcommand{\wt}{\w_{t}}

\renewcommand{\hat}{\widehat}

\newcommand{\maxOp}{\vee}

\newcommand{\abs}[1]{\left|#1\right|}

\newcommand{\Set}[1]{\left\{#1\right\}}

\newcommand{\Max}[1]{\max\Set{#1}}

\newcommand{\minOp}{\wedge}

\newcommand{\eg}{\textit{e.g.}}
\newcommand{\ie}{\textit{i.e.}}

\newcommand{\cA}{\mathcal{A}}

\newcommand{\cH}{\mathbf{H}}
\newcommand{\cI}{\mathcal{I}}

\newcommand{\g}{\bg}
\newcommand{\gt}{\g_{t}}
\newcommand{\bDelta}{\boldsymbol{\Delta}}
\newcommand{\G}{\tilde{\g}}
\newcommand{\W}{\tilde{\w}}
\newcommand{\Cmp}{\tilde{\cmp}}
\newcommand{\Seq}{\mathrm{Seq}}

\newcommand{\bLambda}{\boldsymbol{\Lambda}}

\newcommand{\Gt}{\G_{t}}
\newcommand{\Wt}{\W_{t}}

\newcommand{\Ceil}[1]{\left\lceil#1\right\rceil}

\renewcommand{\Exp}[1]{\exp\brac{#1}}

\newcommand{\wrt}{\textit{w.r.t}}
\newcommand{\cM}{\mathbf{S}}

\newcommand{\cW}{\mathcal{W}}

\newcommand{\EE}[1]{\mathbb{E}\sbrac{#1}}
\newcommand{\e}{\mathbf{e}}
\newcommand{\Prob}[1]{\mathbb{P}\Set{#1}}
\newcommand{\ones}{\boldsymbol{1}}
\newcommand{\et}{\e_{t}}
\newcommand{\cE}{\mathcal{E}}
\newcommand{\sg}{g}
\newcommand{\sgt}{\sg_{t}}
\newcommand{\sw}{w}
\newcommand{\swt}{\sw_{t}}
\newcommand{\pmat}[1]{\begin{pmatrix}#1\end{pmatrix}}
\newcommand{\vecOp}{\text{vec}}
\newcommand{\Diag}[1]{\text{Diag}\brac{#1}}

\newcommand{\diffOp}{\boldsymbol{\Sigma}}
\newcommand{\diffOpScalar}{\Sigma}
\newcommand{\localP}{\bar{P}}
\newcommand{\cmpbar}{\bar{\cmp}}

\newcommand{\Inf}{\mathrm{Inf}}

\newcommand\SmallMatrix[1]{{%
  \tiny\arraycolsep=0.3\arraycolsep\ensuremath{\begin{pmatrix}#1\end{pmatrix}}}}
\usepackage[boxruled,lined]{algorithm2e}
\SetKw{KwRequire}{Require}
\SetKw{KwInput}{Input}
\SetKw{KwInitialize}{Initialize}

\newcommand{\SecLB}{Lower bounds for unconstrained dynamic regret}
\newcommand{\SecApplications}{Dynamic regret for unconstrained OLO via weighted norms}
\newcommand{\SecCoupling}{Recovering Variance-Variability Coupling Guarantees}
\newcommand{\SecSuperlinear}{Adapting to Squared Path-length Requires Superlinear Regret}

\newtheorem{theorem}{Theorem}

\newtheorem{remark}{Remark}

\newtheorem{definition}{Definition}

\definecolor{mydarkblue}{rgb}{0,0.08,0.45}
\hypersetup{ %
  pdftitle={},
  pdfsubject={Proceedings of the International Conference on Machine Learning 2024},
  pdfkeywords={},
  pdfborder=0 0 0,
  pdfpagemode=UseNone,
  colorlinks=true,
  linkcolor=mydarkblue,
  citecolor=mydarkblue,
  filecolor=mydarkblue,
  urlcolor=mydarkblue,
  }

\usepackage{cleveref}

\title{An Equivalence Between Static and Dynamic Regret Minimization}

\author{%
  Andrew Jacobsen\thanks{Work done while visiting Optimal Lab at KAUST.}\\
  Universit\`{a} degli Studi di Milano\\
  Politecnico di Milano\\
  \texttt{contact@andrew-jacobsen.com} \\
  \And
  Francesco Orabona\\
  KAUST\\
  \texttt{francesco@orabona.com}
}
\date{}

\begin{document}

\maketitle

%auto-ignore
\begin{abstract}

We study the problem of dynamic regret minimization in online convex
optimization, in which the objective is to minimize the difference between the
cumulative loss of an algorithm and that of an arbitrary sequence of
comparators. While the literature on this topic is very rich, a unifying
framework for the analysis and design of these algorithms is still missing. In
this paper we show that \emph{for linear losses, dynamic regret minimization is
  equivalent to static regret minimization in an extended decision space}. Using
this simple observation, we show that there is a frontier of lower bounds
trading off penalties due to the variance of the losses and penalties due to
variability of the comparator sequence, and provide a framework for achieving
any of the guarantees along this frontier. As a result, we also prove for the first
time that adapting to the squared path-length of an arbitrary sequence of
comparators to achieve regret
$R_{T}(\cmp_{1},\dots,\cmp_{T})\le \scO(\sqrt{T\sum_{t} \|\cmp_{t}-\cmp_{t+1}\|^{2}})$
is impossible. However,
using our framework we introduce an alternative notion of
variability based on a locally-smoothed comparator
sequence $\bar \cmp_{1}, \dots, \bar \cmp_{T}$, and provide an algorithm guaranteeing dynamic regret of the form
$R_{T}(\cmp_{1},\dots,\cmp_{T})\le \tilde \scO(\sqrt{T\sum_{i}\|\bar \cmp_{i}-\bar \cmp_{i+1}\|^{2}})$, while still matching in the worst case the usual path-length dependencies up to polylogarithmic terms.
\end{abstract}

%auto-ignore

\section{Introduction}%
\label{sec:intro}

This paper introduces new techniques for \emph{Online Convex Optimization} (OCO),
a framework for designing and analyzing algorithms which
learn on-the-fly from a stream of data \citep{Gordon99b,zinkevich2003online,Cesa-BianchiL06,Orabona19, Cesa-BianchiO21}.
Formally, consider $T$ rounds of interaction between the learner and their
environment. In each round, the learner chooses $\wt\in \cW$ from
a convex feasible set $\cW\subseteq\R^{d}$, the environment reveals
a $G$-Lipschitz convex loss function $\ell_{t}:\cW\to\R$, and the learner
incurs a loss of $\ell_{t}(\wt)$.
The classic objective in this setting is to minimize the
learner's \emph{regret} relative to any fixed benchmark $\cmp\in \cW$:
\begin{equation*}
  R_{T}(\cmp) := \sumtT (\ell_{t}(\wt)-\ell_{t}(\cmp))~.
\end{equation*}
In this paper, we study the more general problem of minimizing the learner's
regret relative to any \emph{sequence}
of benchmarks $\cmp_{1},\dots,\cmp_{T}\in \cW$~\citep{HerbsterW98b,HerbsterW01}:
\begin{equation*}
  R_{T}(\cmp_{1},\dots,\cmp_{T})
  :=\sumtT (\ell_{t}(\wt)-\ell_{t}(\cmp_{t}))~.
\end{equation*}
This objective is typically referred to as \emph{dynamic}
regret, to distinguish it from the special case where
the comparator sequence is fixed $\cmp_{1}=\dots=\cmp_{T}$ (referred to as
\emph{static} regret).
We focus in particular on the special case of \emph{Online Linear Optimization} (OLO),
in which $\ell_{t}(\w)=\inner{\gt,\w}$ for some $\gt \in \R^d$. Note
that OCO problems can always be reduced to OLO via
the well-known inequality $\ell_{t}(\wt)-\ell_{t}(\cmp)\le \inner{\gt,\wt-\cmp}$
for $\gt\in\partial\ell_{t}(\wt)$, where $\partial \ell_{t}(\wt)$ is the
subdifferential set of $\ell_t$ at $\wt$  \citep[see,
\textit{e.g.},][]{shalev2011online}, so throughout this paper we
will focus on the OLO setting.

Intuitively, if the sequence of comparators $\cmp_1, \dots, \cmp_T$ varies
too much, it should be impossible to achieve low dynamic regret. %
On the other hand, we know it is possible to achieve sublinear regret if the
sequence of comparators is constant, \emph{i.e.}, $\cmp_1=\dots=\cmp_T$, because
this is simply the static case.
Hence, we need a way to quantify the complexity, or \emph{variability}, of the
comparator sequence. The most commonly used notion of complexity in this regard is the
\emph{path-length} of the comparator sequence \citep{HerbsterW98b,HerbsterW01}, defined as
\[
  P^{\|\cdot\|}_T 
  := \sum_{t=2}^T \|\cmp_t-\cmp_{t-1}\|~.
\]
It is possible to show that Online Gradient Descent has a dynamic regret of $\scO((D+P_T^{\|\cdot\|}) G \sqrt{T})$ in \emph{bounded} domains, where $D$ is an upper bound on the diameter of the feasible set and $G$ is the Lipschitz constant of the losses~\citep{zinkevich2003online}.
This bound was improved to $\scO(\sqrt{D P_T^{\|\cdot\|}} G \sqrt{T})$ and shown to be minimax optimal by \citet{zhang2018adaptive}.

Notice that the path-length bounds scale with a rather pessimistic constant of $D=\sup_{w,w'\in\cW}\norm{\w-\w'}$.
A better bound would instead scale with the \emph{squared} path-length:
\[
P^{\|\cdot\|^2}_T 
:= \sum_{t=1}^{T-1} \|\cmp_t-\cmp_{t-1}\|^2,\label{eq:sqr-pl}
\]
which can be significantly smaller\footnote{Note that the
  bound of
  \citet{zhang2018adaptive} trivially implies a squared path-length dependence
  since $P_{T}^{\norm{\cdot}}\le \sqrt{TP_{T}^{\norm{\cdot}^{2}}}$, so
  one can obtain a bound of
  $\scO(\sqrt{DP_{T}^{\norm{\cdot}}}G\sqrt{T})\le \scO(D^{1/2}(P_{T}^{\norm{\cdot}^{2}})^{1/4}GT^{3/4})$.
  However, this
  bound is not interesting because it does not remove the dependence on $D$ and
  it is never better than the existing
  $\sqrt{DP_{T}^{\norm{\cdot}}}G\sqrt{T}$ bound.} than the penalty in the bound above:
$P_{T}^{\|\cdot\|^{2}}\le  D P_{T}^{\|\cdot\|}$.
However, guarantees scaling with $P_{T}^{\norm{\cdot}^{2}}$ are not well
understood in general compared with the more common $P_{T}^{\norm{\cdot}}$
bounds, and have only been obtained by restricting
the comparator sequence to $\cmp_{t}=\argmin_{\w\in\cW} \ \ell_{t}(\w)$
or under additional assumptions such as strong-convexity
\citep{yang2016tracking,zhang2017improved,chang2021online}.

In this paper, we focus on the challenging case that the domain is \emph{unbounded}, where recent works have achieved the dynamic regret 
$\tilde\scO\big(\sqrt{\max_{t,t'}\norm{\cmp_{t}-\cmp_{t'}}P_{T}^{\|\cdot\|}T}\big)$ in the
worst case
\citep{jacobsen2022parameter,luo2022corralling,jacobsen2023unconstrained,zhang2023unconstrained}.
Of particular interest, \citet{jacobsen2022parameter,zhang2023unconstrained}
achieve bounds of the form
\begin{align}
  R_{T}(\cmp_{1},\dots,\cmp_{T})\le \tilde \scO\brac{\sqrt{P_{T}^{\norm{\cdot}}{\textstyle\sumtT} \norm{\gt}^{2}\norm{\cmp_{t}-\bar{\cmp}}}},\label{eq:intro:coupled-regret}
\end{align}
which avoids the pessimistic multiplicative penalty of
$\max_{t,t'}\norm{\cmp_{t}-\cmp_{t'}}$, but results in a coupling
between the gradient and variability
penalties. It is unclear if it is possible to obtain
a guarantee which cleanly separates the variability and
variance penalties, to achieve dynamic regret scaling as
$R_{T}(\cmp_{1},\ldots,\cmp_{T})\le\scO\brac{\sqrt{P_{T}^{\norm{\cdot}^{2}}\sumtT\norm{\gt}^{2}}}$.
In fact, it is not clear
in general how to reason about potential trade-offs that may result
from adapting to different measures of variability.

\paragraph{Contributions.}
In this paper, we show how to reformulate the dynamic regret miniminization
problem as an \emph{equivalent} static regret problem (\Cref{sec:reduction}). This equivalence allows us to use results for the static regret setting to prove both upper and lower bounds for dynamic regret.

In our first application of this equivalence, we show that the ideal
guarantee scaling the with squared path-length
$\scO\big(\sqrt{P_{T}^{\|\cdot\|^{2}}\sumtT \norm{\gt}^{2}}\big)$ is \textbf{not}
possible in general (\Cref{sec:lb}).
We do this by proving a novel lower bound showing that there is a
fundamental trade-off between the penalties incurred due to comparator
variability and penalties incurred due to loss variance, leading to a
new frontier of dynamic regret lower bounds.

Our second application is to provide a
framework for achieving any of the variance/variability trade-offs along the lower bound frontier, up to
polylogarithmic terms (\Cref{sec:applications}).
Our framework allows us to develop dynamic regret algorithms by simply choosing
suitable dual-norm pairs $(\norm{\cdot},\norm{\cdot}_{*})$ in the static regret problem.
Along with our matching lower bound, this framework provides a concrete way to
reason about different measures of comparator variability and the trade-offs
they entail, and to design algorithms achieving those trade-offs.

While our lower bound demonstrates that the ideal squared path-length guarantee cannot be achieved,
using our framework we show that it is possible to
achieve an alternative guarantee that scales with
\begin{align*}
  \bar{P}^{\norm{\cdot}^{2}}(\cmp_{1},\ldots,\cmp_{T})\approx \sum_{i}\norm{\cmpbar_{i}^{(\tau)}-\cmpbar_{i+1}^{(\tau)}}^{2}_{2},
\end{align*}
where $\bar \cmp_{i}^{(\tau)}$ is a \emph{local average} of the comparator sequence at
a timescale of $\tau$ (see \Cref{sec:haar}). Similar to $P_{T}^{\|\cdot\|^{2}}$, this variability measure
maintains the property that it matches the worst-case guarantees based on path-length
up to polylogarithmic terms, \ie{}, $\bar P_{T}^{\|\cdot\|^{2}}\le \tilde \scO(\max_{t,t'}\norm{\cmp_{t}-\cmp_{t'}}P_{T}^{\|\cdot\|})$.
These are the first guarantees for general OCO that fully decouple the variance and variability
penalties for dynamic regret without explicitly incurring
pessimistic $\max_{t,t'}\norm{\cmp_{t}-\cmp_{t'}}$ penalties.

  %auto-ignore
\paragraph{Related Work.}

Our approach is inspired by the Haar OLR algorithm of
\citet{zhang2023unconstrained}. In that work,
dynamic regret is approached by interpreting
the comparator sequence as a high-dimensional ``signal''
which is decomposed into a frequency domain
representation using a dictionary of features.
Then, for each feature vector in the dictionary,
a 1-dimensional parameter-free~\citep{orabona2016coin, MhammediK20} algorithm is used
to learn how well that feature correlates with the losses.
This allows one to compete with an arbitrary comparator sequence,
so long as it can be represented in terms of the chosen dictionary of features.
We take a similar but slightly more general approach. Our framework
also represents the comparator sequence as a high-dimensional signal,
but we instead use this signal to define an \emph{equivalent static regret problem},
a perspective that allows us design algorithms for dynamic regret
by simply choosing suitable dual-norm pairs.

Other prior works in the general OCO setting have also
studied various alternative forms of variability
such as the temporal variability $\sumtTmm\sup_{\bw\in\cW}\abs{\ell_{t}(\bw)-\ell_{\tpp}(\bw)}$
\citep{besbes2015nonstationary,jadbabaie2015online,campolongo2021closer} or
deviation of the comparator from a given dynamical model
$\sumtTmm\norm{\cmp_{t}-\Phi_{t}(\cmp_{\tmm})}$ \citep{hall2016online}.
Alternative variance penalties have also been studied in the dynamic setting, such as the small-loss
penalties $\sumtT \ell_{t}(\cmp_{t})$ or gradient variation penalties
$\sumtT \sup_{\bw\in\cW}\norm{\grad\ell_{t}(\bw)-\grad\ell_{\tpp}(\bw)}$ \citep{gyorgy2016shifting,zhao2022efficient,jacobsen2023unconstrained,zhao2024adaptivity}. It is also possible to achieve a smaller regret with stronger assumptions on the losses~\citep{baby2021optimal}.
It is important to note however that almost all prior works, with the exception of \citet{jacobsen2022parameter,luo2022corralling} and \citet{zhang2023unconstrained}, study dynamic regret only in the easier bounded domain setting.

There is also an often ignored connection between measures of comparator
variability and the function classes studied in non-parametric regression. In
particular, considering the case that the losses are $\ell_t(x) = (x-u_t)^2$, then
the sequence of comparators $u_1, \dots, u_T$ with bounded path length $C_T$ and
bounded squared path length $(C'_T)^{2}$ corresponds to the sequence with discrete
total variation bounded by $C_T$ and the discrete Sobolev class with bound $C'_T$, respectively. In this setting, the minimax rates are known \cite{KoolenMBAY15,SadhanalaWT16}
and \citet{KoolenMBAY15} obtain the minimax regret for the Sobolev classes, while \citet{BabyW19} for both classes with slightly stronger assumptions. However, these results are not directly related to this paper because we consider linear losses.

  %auto-ignore

\paragraph{Notations.}
We will use the following definitions and notations. The elements of a matrix $\bA \in \R^{n \times m}$ are denoted by $A_{ij}$ for $i=1,\dots,n$ and $j=1, \dots,m$. Similarly, the elements of a vector $\bu \in \R^d$ are $u_i$ for $i=1, \dots,d$.
The Kronecker product of
matrices $\bA\in\R^{m\times n}$ and $\bB\in\R^{p\times q}$ is the block matrix
defined by
\begin{align*}
  \bA\otimes \bB 
  := \SmallMatrix{A_{1,1}\bB& \dots &A_{1,n}\bB\\
  \vdots&\ddots&\vdots\\
  A_{m,1}\bB&\dots&A_{m,n}\bB
  }~.
\end{align*}
We let $\e_{t}$ denote the $t^{\text{th}}$ standard basis vector of $\R^{T}$ and
$\bI_{d}$ is the $d\times d$ identity matrix. For a square matrix $\bA$, $\Diag{\bA}$ is the diagonal matrix that contains the elements of the diagonal of $\bA$.
For a positive definite matrix $\bM$, we define the weighted norm $\norm{\bx}_{\bM}:=\sqrt{\inner{\bx,\bM \bx}}$. For a matrix $\bA \in \R^{m \times n}$, we denote its Frobenius norm by $\|\bA\|_F:= \sqrt{\sum_{i=1}^m \sum_{j=1}^n A^2_{i,j}}$. The $\vecOp$ operator is the mapping defined by stacking the columns of a matrix $\bA$ in a vector.
We will denote by $\|\bA\|_{p,p}$ the entry-wise $p$-norm of $\bA$, \textit{i.e.}, $\|\bA\|_{p,p}:=\|\vecOp(\bA)\|_p$.

%auto-ignore

\section{A dynamic-to-static reduction}
\label{sec:reduction}

In this section, we present a general reduction from
dynamic regret to static regret.
The key idea is
to embed the comparator
sequence in a high dimensional space $\cW^{T}$, where $T$ is the number of rounds, so that
competing with a \emph{fixed} comparator $\Cmp\in \cW^{T}$ in this
high-dimensional space is equivalent to
competing with a \emph{sequence} of comparators in the
original space $\cW$.
In this way, we can reduce the problem of minimizing the dynamic regret to the one of minimizing the static regret. 

\SetAlCapHSkip{0.5em}
\setlength{\algomargin}{0.5em}
\begin{algorithm}[t!]
  \SetAlgoLined
  \KwInput{Domain $\cW\subseteq\R^{d}$, Online Learning Algorithm $\cA$ with
    domain $\cW^{T}$}\\
  \For{$t=1:T$}{
    Get $\Wt=(\wt^{(1)},\ldots,\wt^{(T)})\in\cW^{T}$ from $\cA$\\
    Play $\wt = \wt^{(t)}\in\cW$ and observe $\gt\in\partial\ell_{t}(\wt)$\\
    Pass $\Gt=\et \otimes \gt = (\zeros_d^\top,\dots,\zeros_d^\top,\underbrace{\gt^\top}_{\mathclap{\text{indices }i\in[d(t-1)+1,dt]}},\zeros_d^\top,\dots)^\top\in\cW^{T}$ to $\cA$\\
  }
  \caption{Dynamic-to-Static Reduction}
  \label{alg:redux}
\end{algorithm}

Our reduction is shown in \Cref{alg:redux}.
We simply embed the linear losses $\gt$ in a high-dimensional
space by setting
\begin{align}
  \Gt &=
  \et\otimes \gt = (\zeros_{d}^\top, \dots,\zeros_{d}^\top,\underbrace{\gt^{\top}}_{\mathclap{\text{Indices }\in [d(t-1)+1, dt]}},\zeros_{d}^\top,\dots,\zeros_{d}^\top)^{\top},\label{eq:Gt}
\end{align}
where $\et\in\R^{T}$ is the $t^{\text{th}}$ standard basis vector of $\R^{T}$ and
$\zeros_{d}\in\R^{d}$ denotes the vector of zeros. We pass these losses to the
online learning algorithm $\cA$, which predicts with a vector $\Wt\in\cW^{T}$.
Finally, we set $\bw_t\in\R^{d}$ equal to the $t^{\text{th}}$ ``component'' of
$\Wt$, and play $\wt$.

We show that the dynamic regret of the resulting algorithm will be equal to the static regret of the algorithm $\cA$.
In particular, 
for any sequence $\vec{\cmp}=(\cmp_{1},\dots,\cmp_{T})$ in $\cW$ we will
denote the concatenation of $\vec{\cmp}$ into a single vector in $\cW^{T}$ as
\begin{align}
  \Cmp = {\textstyle\sumtT }\et\otimes\cmp_{t}%
  =(\cmp_{1}^\top, \dots, \cmp_{T}^\top)^\top~.\label{eq:Cmp}
\end{align}
Then, the following proposition shows that the dynamic regret of \Cref{alg:redux} \wrt{} any sequence $\vec{\cmp}=(\cmp_{1},\ldots,\cmp_{T})$ is \emph{equal} to the static regret of $\cA$ \wrt{} $\Cmp$.
\begin{restatable}{proposition}{DynamicToStatic}\label{prop:dynamic-to-static}
  Let $\cW\subseteq\R^{d}$ and let $\cA$ be an online learning algorithm with
  domain $\cW^{T}$.
  Then, for any sequence $\vec{\cmp}=(\cmp_{1},\ldots,\cmp_{T})\in\cW^{T}$,
  \Cref{alg:redux} guarantees
  \begin{align*}
    R_{T}(\vec{\cmp})=\sumtT \inner{\gt,\wt-\cmp_{t}} = \sumtT \inner{\Gt,\Wt-\Cmp} =: R_{T}^{\Seq}(\Cmp)~.
  \end{align*}
\end{restatable}
\begin{proof}
  The proof is immediate from \Cref{eq:Gt,eq:Cmp}.
  In fact, observe that the cumulative loss of the comparator sequence
  is precisely
  \begin{align*}
    \sumtT \inner{\gt,\cmp_{t}}= \inner{\SmallMatrix{\g_{1}\\\vdots\\\g_{T}},\SmallMatrix{\cmp_{1}\\\vdots\\\cmp_{T}}}
    =\inner{\sumtT \et\otimes\gt, \sumtT \et\otimes\cmp_{t}}=\inner{\sumtT \Gt,\Cmp}~.
  \end{align*}
  We get a similar relationship for the algorithm's cumulative loss.
  Hence,
  we have
  $R_{T}(\vec{\cmp})=\sumtT\inner{\gt,\wt-\cmp_{t}}=\sumtT \inner{\Gt,\Wt-\Cmp}=R_{T}^{\Seq}(\Cmp)$.
\end{proof}
\begin{remark}
  It is important to note that the regret equivalence holds in the
  context of
  \textbf{linear losses}. However, our reduction can still be leveraged
  for arbitrary convex losses by first applying the standard reduction
  to OLO:
  $\sumtT \ell_{t}(\wt)-\ell_{t}(\cmp_{t})\le \sumtT \inner{\gt,\wt-\cmp_{t}}=R_{T}^{\Seq}(\Cmp)$
  for $\gt\in\partial\ell_{t}(\wt)$.
\end{remark}

While our reduction is exceptionally simple, its utility should not be understated.
\Cref{prop:dynamic-to-static} is a regret \emph{equivalence} --- we lose nothing
by taking this perspective, yet it allows us to immediately apply all the usual techniques
and approaches from the static regret setting. 
For instance,
given any dual norm pair $(\norm{\cdot}, \norm{\cdot}_{*})$, it is well-understood how to develop algorithms which adapt
simultaneously
to the comparator norm $\norm{\Cmp}$ and to the gradient variance
$\sumtT \norm{\Gt}^{2}_{*}$
to guarantee
\begin{align*}
  R_{T}(\cmp_{1},\dots,\cmp_{T})=R_{T}^{\Seq}(\Cmp)\le \tilde\scO\brac{\norm{\Cmp}\sqrt{\textstyle{\sumtT} \norm{\Gt}^{2}_{*}}}.
\end{align*}
Such algorithms are commonly referred to as ``parameter-free'', or ``comparator
adaptive'', because they  achieve this adaptation by completely removing
the parameter that depends on the unknown comparator $\Cmp$~\citep[\emph{e.g.},][]{mcmahan2012noregret,mcmahan2014unconstrained,orabona2016coin,cutkosky2018black,foster2018online,jacobsen2022parameter,ChenCO22,jacobsen2023unconstrained,zhang2023improving}.
In this way, we have effectively reduced the problem of minimizing dynamic regret to
the problem of selecting a dual-norm pair
$(\norm{\cdot},\norm{\cdot}_{*})$ that meaningfully measures the
``difficulty'' of the sequence in $\Cmp$ and the losses $\Gt$.
In particular, $(\norm{\cdot},\norm{\cdot}_{*})$
should be chosen with the following
considerations in mind:
\begin{enumerate}
        \item $\norm{\Cmp}$ should produce a meaningful measure of variability
        of the comparator sequence $\cmp_{1},\ldots,\cmp_{T}$.
        For instance,
        we will show in \Cref{prop:sqr-trade-off} that
        the squared path-length
        arises from a particular weighted norm applied to $\Cmp$.
        \item $\norm{\Gt}_{*}$ should not ``blow up''. Ideally
        $\norm{\Gt}_{*}$ should match the magnitude of the true losses
        $\gt$ up to polylog factors.
  \item $(\norm{\cdot},\norm{\cdot}_{*})$ should be chosen with computational
        considerations in mind. For instance,
        to apply an FTRL-based algorithm to the losses $\Gt\in\R^{dT}$,
        efficient implementation will typically require
        $\norm{\cdot}_{*}$ to have sparse subgradients.
        In general, an ideal dual-norm pair should facilitate
        updating only
        $\scO(\log T)$ variables at a time, so as to match the $\scO(d\log T )$
        per-step computation enjoyed by existing dynamic regret algorithms.
        We will see one such example in \Cref{sec:haar}.
\end{enumerate}
In the next section, we show that there is in fact a fundamental
trade-off between the penalties induced by the
dual-norm pair $(\norm{\cdot},\norm{\cdot}_{*})$, creating a tension between
the first two considerations.

%auto-ignore

\section{\SecLB}
\label{sec:lb}

In the static regret setting, there is a well-known trade-off between the way in which we measure the complexity of the comparator $\cmp$ and the way in which we measure the complexity of the linear losses $\gt$. For example, in Online Mirror Descent~\citep{NemirovskijY83,WarmuthJ97} one can get a regret guarantee that depends on the maximum diameter of the feasible set with respect to a norm $\|\cdot\|$, while the linear losses are measured using the dual norm $\|\cdot\|_*$.
The equivalence in \Cref{prop:dynamic-to-static} suggests that  a similar tension exists for the dynamic regret.

Given the structure of our reduction, it makes sense to focus on the weighted norms $\norm{\cdot}_{\bM}$ and $\norm{\cdot}_{\bM^\inv}$, where $\bM$ is a
symmetric positive definite matrix.
In particular, the next theorem shows that there is a fundamental trade-off between a
\emph{variability penalty} $\norm{\Cmp}_{\bM}$ and a \emph{variance penalty}
$G^{2}\Tr(\bM^{\inv})$ related to the losses.
The proof is provided in \Cref{app:pf-lb} and it is based on a lower bound to the tail of Rademacher chaos of order 2.

\noindent
\begin{minipage}{\columnwidth}
\begin{restatable}{theorem}{PFLB}\label{thm:pf-lb}
  Let the number of rounds $T\geq T_0$, where $T_0$ is a universal constant. Let $\cA$ be an online learning algorithm, and suppose $\cA$ guarantees $R_{T}(0)\le G\epsilon_{T}$ for any sequence of linear losses
  $\sg_{1},\dots,\sg_{T}\in\R$ satisfying $\abs{\sgt}\le G$.
  Let $\bM^{\inv}\in\R^{T\times T}$  be any symmetric positive definite matrix,
  denote $\tilde\bM^{\inv}:=\bM^{\inv}-\Diag{\bM^{\inv}}$ and
  $V_{T} := \Tr(\bM^{\inv})+\|\tilde\bM^{\inv}\|_{F}$.
  Suppose
  that
  $\|\tilde\bM^{\inv}\|^{2}_{F}\ge \frac{T}{2}\max_{i}\sum_{j}(\tilde M^{\inv}_{ij})^{2}$.
  Then, for any $P$ satisfying
  $T_{0}\le \log_2 \frac{\sqrt{PV_{T}}}{2\epsilon_{T}}\le T$, there is a sequence of losses $g_{1},\dots,g_{T}\in\R$, and $\Cmp=(u_{1},\dots,u_{T})^{\top}\in\R^{T}$ satisfying $\norm{\Cmp}_{\bM}= \sqrt{P}$ such that we have
  \begin{align*}
    R_{T}(u_{1},\ldots,u_{T})
    &\ge
      \Omega\brac{G\epsilon_{T}+G\sqrt{P\sbrac{\Tr(\bM^{\inv})+\norm{\tilde\bM^{\inv}}_{F}\log_{2}^{\half}\frac{\sqrt{PV_{T}}}{2\epsilon_{T}}}}}~.
  \end{align*}
\end{restatable}
\end{minipage}

Let us first briefly discuss the conditions on $\bM$.
First, note that the restriction that $\bM$ be positive definite and symmetric
simply specifies that $\norm{\cdot}_{\bM}$ defines a valid norm. The condition
on $\|\tilde{\bM}^\inv\|_F=\|\bM^\inv-\Diag{\bM^\inv}\|_{F}$ is less straight forward
to interpret, but
it essentially states
that the total ``variance'' of $\tilde\bM^\inv$ is
at least as much as that of any of its columns.
on a technical level this assumption leads to the restriction that $P$
satisfies
$\log_{2}\brac{\sqrt{PV_{T}}/2\epsilon_{T}}\le T$. This is a natural restriction
which encodes the fact that if $P$ is too large relative to $T$
(\textit{i.e.}, when $\log_{2}(\sqrt{PV}/2\epsilon_{T})\ge T$),
one can ensure ``low'' regret by
simply playing $\wt=\zeros$ on every round:
\begin{align*}
  R_{T}(\vec{\cmp})
  =-{\textstyle\sumtT} \inner{\gt,\cmp_{t}}
  \le \max_{t}\norm{\cmp_{t}} G\,T
  \le G\max_{t}\norm{\cmp_{t}}\log_{2}\brac{\sqrt{PV_{T}}/2\epsilon},
\end{align*}
and hence the only lower bounds in such settings are trivial ones and it
suffices to consider only $P$ satisfying $\Log{\sqrt{PV_{T}}/2\epsilon}\le T$.
We will see in \Cref{prop:sqr-trade-off} that the matrix that produces the squared path-length
satisfies this condition, and it can be seen that symmetric matrices
with equal column sums
(as is the case in \Cref{prop:haar-trade-off}) satisfy this condition as well.

The result of \Cref{thm:pf-lb} shows that there is a frontier of lower bounds which trade off penalties
related to variability of the comparator sequence and penalties related to the
variance of the subgradients. That is,
one can not guarantee a small variability penalty in all situations without also accepting
a large subgradient variance penalty.
The next proposition shows that i) the squared path-length can be represented
by a particular choice of the weighted norm $\norm{\Cmp}_{\bM}$, and ii) the fundamental
tension between $\norm{\Cmp}_{\bM}$ and its corresponding variance penalty
$\Tr(\bM^{\inv})$ prevents any algorithm from attaining the ideal variability dependence
of $\norm{\Cmp}_{\bM}=\sc\Big(\sqrt{\sumtTmm\|\cmp_{t}-\cmp_{\tpp}\|^{2}}\Big)$. In fact, the corresponding
variance penalty is $G^{2}\Tr(\bM^{\inv})=\scO(G^{2}T^{2})$, resulting in a vacuous guarantee.
Proof of the proposition can be found in \Cref{app:sqr-trade-off}.
\begin{restatable}{proposition}{SqrTradeOff}\label{prop:sqr-trade-off}
  \textbf{(\SecSuperlinear)}
  Define the \emph{finite-difference} operator
  $\diffOp\in\R^{T}$ as the matrix with entries
  \begin{align*}
    \diffOpScalar_{ij}=\begin{cases}1&\text{if }i=j\\
            -1&\text{if }i=j-1\\
            0&\text{otherwise}
          \end{cases}.
    \end{align*}
    Let $\cM=\diffOp^{\top}\diffOp$ and $\bM=\cM\otimes \bI_d$. Then,
    $\bM$ satisfies the assumptions of \Cref{thm:pf-lb} and
    \begin{align*}
      \norm{\Cmp}_{\bM}^{2}
      =
      \norm{\cmp_{T}}^{2}_{2}+\sumtTmm\norm{\cmp_{t}-\cmp_{\tpp}}^{2}_{2}
      \qquad\text{ and }\qquad
      \Tr\brac{\bM^{\inv}} = \frac{T(T+1)}{2}~.
    \end{align*}
\end{restatable}
\Cref{prop:sqr-trade-off} shows that adapting to
the squared path-length of an arbitrary comparator sequence \emph{necessarily
requires} incurring a linear penalty, so adapting to the
squared path-length is impossible without facing a vacuous guarantee. However,
we will show in \Cref{sec:haar} that it is possible to adapt to
a measure of variability which is similar in spirit to the
squared path-length, yet only incurs a $\Tr(\bM^{\inv})=\scO(\log T)$ variance penalty.

\begin{remark}
  The matrix $\bM$ in \Cref{prop:sqr-trade-off}
  uniquely exposes the the squared path-length up to the bias term
  $\norm{\cmp_{T}}^{2}$. Such a bias term must appear because
  in the static regret setting, wherein $\cmp_{1}=\ldots=\cmp_{T}=\cmp$,
  the variability measure $\norm{\cdot}_{M}$ must still reduce to a dependence on
  $\norm{\cmp}$,
  otherwise the guarantee would violate
  existing $\tilde\Omega(\norm{\cmp}\sqrt{T})$ lower bounds for static regret.
  More generally, we show in \Cref{app:lb-details} that
  any other choice of bias would similarly lead to
  $\Tr\brac{\bM^{\inv}}\ge \Omega(T^{2})$, so \Cref{prop:sqr-trade-off} along with
  our lower bound in \Cref{thm:pf-lb} are
  sufficient to show that adapting to squared path-length requires accepting a
  vacuous guarantee.
\end{remark}

%auto-ignore
\section{\SecApplications}%
\label{sec:applications}

So far, we've seen that there exists a frontier of
lower bounds trading off a variability penalty,
measured by $\norm{\Cmp}_{\bM}$, and a loss variance
penalty, measured by $\Tr(\bM^{\inv})$, and that
the tension between these two quantities
makes it impossible to adapt to the squared path-length of the
comparator sequence without accepting a vacuous regret guarantee.
A natural next question is whether there are choices of
$\bM$ which lead to a more favorable trade-off of
these two quantities.
In this section, we provide a simple framework for achieving
lower bounds along the frontier described by \Cref{thm:pf-lb},
and an instance which successfully achieves an improved variance/variability trade-off.
The guarantees on the lower bound frontier can be achieved using
any parameter-free algorithm along with
the 1-dimensional reduction of \citet{cutkosky2018black}
to extend the algorithm to dual-norm pair
$(\norm{\cdot}_{\bM},\norm{\cdot}_{\bM^{\inv}})$. The generic procedure is summarized in \Cref{alg:1d-redux} for convenience.

\begin{algorithm}[t]
  \SetAlgoLined
  \KwInput{1-d Parameter-free OLO algorithm $\cA$,
    positive definite symmetric matrix $\bM\in\R^{dT\times dT}$}\\
  \KwInitialize{$\tilde{\btheta}_{1}=\tilde{\bv}_{1}=\zeros\in\R^{dT}$, $V_{1}=0$}\\
  \For{$t=1:T$}{
    Get $\beta_{t}\in\R$ from $\cA$\\
    Play $\Wt = \beta_{t}\tilde{\bv}_{t}$ and observe $\Gt$\\
    Send $\inner{\tilde{\bv}_{t},\Gt}$ to $\cA$ as the $t^{\text{th}}$ loss\\
    \BlankLine
    \BlankLine
    Set
    $\tilde{\btheta}_{\tpp} = \tilde \btheta_{t}- \bM^{\inv}\Gt$\\%\hfill\tcp{Update
    Set $V_{\tpp}=V_{t}+\norm{\Gt}^{2}_{\bM^{\inv}}$\\%\hfill\tcp{Update scale factor}
    Set $\tilde{\bv}_{\tpp}=\frac{\tilde\btheta_{\tpp}}{\sqrt{V_{\tpp}}}\sbrac{1\minOp\frac{ \sqrt{V_{\tpp}}}{\norm{\tilde\btheta_{\tpp}}_{\bM^{\inv}}}}$\hfill\tcp{
      (Projected) Scale-free
  FTRL update}
  }
  \caption{Dynamic regret OLO through 1-dimensional reduction \citep{cutkosky2018black}}
  \label{alg:1d-redux}
\end{algorithm}

\begin{restatable}{theorem}{SimpleDynamic}\label{thm:simple-dynamic}
  Let $\cM\in\R^{T\times T}$ be a symmetric positive definite matrix, $\bM=\cM\otimes \bI_d$, and $\epsilon>0$.
  There is an algorithm $\cA$ such that for any $\g_{1},\ldots,\g_{T}\in\R^{d}$
  satisfying $\norm{\gt}_{2}\le G$ for all $t$ and any sequence
  $\vec{\cmp}=(\cmp_{1},\ldots,\cmp_{T})\in\R^{dT}$, the dynamic regret is bounded as
  \begin{align*}
    R_{T}(\vec{\cmp})
    &\le
      \scO\brac{\mathfrak{G} \epsilon + \norm{\Cmp}_{\bM}\sbrac{\sqrt{V_{T}\Log{\frac{\norm{\Cmp}_{\bM}\sqrt{V_{T}}}{\mathfrak{G}\epsilon}+1}}\maxOp\mathfrak{G}\Log{\frac{\norm{\Cmp}_{\bM}\sqrt{V_{T}}}{\epsilon\mathfrak{G}}}}},
  \end{align*}
  where $V_{T}=\sumtT \norm{\Gt}^{2}_{\bM^{\inv}}$ and
  $\mathfrak{G}=G\norm{\cM^{\inv}}_{\infty,\infty}$.%
\end{restatable}

For the proof, we will need the following technical lemma.
\begin{restatable}{lemma}{GradientBound}\label{lemma:gradient-bound}
  Let $\cM\in\R^{T\times T}$ be a symmetric positive definite matrix
  and let $\bM=\cM\otimes \bI_d$. For $t=1,\dots,T$, let $\gt\in\R^{d}$ and let $\Gt=\et\otimes\gt$.
  Then, we have $\norm{\Gt}_{\bM}^{2}= \norm{\gt}^{2}_{2}S_{tt}$.
\end{restatable}
\begin{proof}
Using the mixed-product property
  $(A\otimes B)(C\otimes D)=AC\otimes BD$ and the transpose property
  $(A\otimes B)^{\top}=A^{\top}\otimes B^{\top}$ of the Kronecker product, we have that
  \begin{align*}
    \inner{\Gt,\bM\Gt}
    &=
      \inner{\et\otimes \gt, \sbrac{\cM\otimes \bI_d}\et\otimes\gt}
    =
      \inner{\et\otimes \gt,\cM\et\otimes\gt}
    =
      (\et^{\top}\otimes\gt^{\top}) (\cM\et\otimes\gt)\\
    &=
      \et^{\top}\cM\et\otimes \gt^{\top}\gt
    =
      S_{tt}\norm{\gt}^{2}~. \qedhere
  \end{align*}
\end{proof}

\begin{proof}[Proof of \Cref{thm:simple-dynamic}]
  Applying \Cref{prop:dynamic-to-static}, we have $R_{T}(\vec{\cmp})=\sumtT\inner{\Gt,\Wt-\Cmp}=R_{T}^{\Seq}(\Cmp)$.
  Since $\bM$ is symmetric and positive definite,
  $(\norm{\cdot}_{\bM},\norm{\cdot}_{\bM^{\inv}})$ is a valid
  dual-norm pair. By \Cref{lemma:gradient-bound},
  we have
  $\norm{\Gt}_{\bM^{\inv}}^{2}=\norm{\gt}^{2}_{2} S^{\inv}_{tt}\le G^{2} \norm{\cM^{\inv}}_{\infty,\infty}:=\mathfrak{G}^{2}$.
  Hence,
  let $\cA$ be any algorithm which guarantees a
  parameter-free regret \emph{w.r.t.} $(\norm{\cdot},\norm{\cdot}_{*})$
  on losses satisfying $\norm{\Gt}_{\bM^{\inv}}\le \mathfrak{G}$.
  Note that any parameter-free algorithm can
  be extended to handle arbitrary dual-norm pairs
  by leveraging the one-dimensional reduction of \citet[Section 3]{cutkosky2018black}, that reduces the OLO problem to a unconstrained 1d problem plus an OLO problem in the unitary ball defined by the primal norm.
  For instance, applying \citet[Algorithm 1]{jacobsen2022parameter}
  with the one-dimensional reduction one can
  easily show (see details in \Cref{app:1d-redux})
  \begin{align*}
    R_{T}(\vec{\cmp})
    &\le
      \scO\brac{\mathfrak{G} \epsilon + \norm{\Cmp}_{\bM}\sbrac{\sqrt{V_{T}\Log{\frac{\norm{\Cmp}_{\bM}\sqrt{V_{T}}\Lambda_{T}}{\mathfrak{G}\epsilon}+1}}\maxOp\mathfrak{G}\Log{\frac{\norm{\Cmp}_{\bM}\sqrt{V_{T}}\Lambda_{T}}{\epsilon\mathfrak{G}}}}},
  \end{align*}
  where $V_{T}=\sumtT \norm{\Gt}^{2}_{\bM^{\inv}}$ and
  $\Lambda_{T}=\log^{2}(\sumtT \norm{\Gt}^{2}_{\bM^{\inv}}/\mathfrak{G}^{2})\le \scO(\log^{2}T)$.
\end{proof}
Note in particular that by \Cref{lemma:gradient-bound}, we have
$\sumtT \norm{\Gt}^{2}_{\bM^{\inv}}=\sum_{t=1}^T S^{\inv}_{tt}\norm{\gt}^{2}\le G\sumtT S^{\inv}_{tt}=G\Tr(\cM^{\inv})$, so
this bound matches the lower bound from \Cref{sec:lb}, up to polylogarithmic terms.\footnote{Note that the lower bound is stated for $d=1$, in which case $\Tr(\bS^\inv)=\Tr(\bM^\inv)$.}
Thus, any valid choice of $\bM$ will be on the lower bound frontier of
\Cref{sec:lb}.

  %auto-ignore

\subsection{Trading-off Variance and Variability}%
\label{sec:haar}

Leveraging the algorithm characterized by \Cref{thm:simple-dynamic},
we now show that it is indeed possible to
choose $\bM$ such that $\sumtT \norm{\Gt}^{2}_{\bM^{\inv}}$ is only
$\scO(\Log{T}\sumtT \norm{\gt}^{2})$, in exchange for
a variability penalty which is still similar in spirit to
the squared path-length.

Inspired by the Haar OLR algorithm of
\cite{zhang2023unconstrained}, we apply \Cref{thm:simple-dynamic}
using $\cM=\cH_{n} \cH^{\top}_{n}$, where $\cH_{n}$ is the
unnormalized Haar basis matrix of order $n=\Ceil{\log_{2} T}$. 
The Haar wavelet transform and its basis matrix are common tools 
in the signal processing literature; we recall the basic definitions and facts for convenience in \Cref{app:haar}.
With this choice, we have the following bounds on $\norm{\Cmp}_{\bM}$ and
$\norm{\Gt}_{\bM^{\inv}}^{2}$. The proof can be found in \Cref{app:haar-trade-off}.
\begin{restatable}{proposition}{HaarTradeOff}\label{prop:haar-trade-off}
  Let $n=\log_{2} T$ and $\cH_{n}$ be the unnormalized Haar  basis matrix of
  order $n$.
  For any $\tau\in\Set{2^{i}:i=0,\ldots,\log_{2}T}$, let $N_{\tau}=T/\tau$ and let $\cI_{1}^{(\tau)},\ldots,\cI_{N_{\tau}}^{(\tau)}$ be a partition of $[T]$
  into intervals of length $\tau$. Define the average comparator in interval
  $\cI_{i}^{(\tau)}$ to be
  $\cmpbar_{i}^{(\tau)}=\frac{1}{\tau}\sum_{t\in\cI_{i}^{(\tau)}}\cmp_{t}$, and
  define the \emph{squared path-length at time-scale $\tau<T$} to be
  \[
    \localP(\vec{\cmp},\tau) 
    := \sum_{i=1}^{ N_{\tau}/2}\norm{\cmpbar_{2i-1}^{(\tau)}-\cmpbar_{2i}^{(\tau)}}^{2}_{2},
  \]
  and $\localP(\vec{\cmp},T)=\norm{\cmpbar^{(T)}_{1}}^{2}_{2}=\norm{\cmpbar}^{2}_{2}$.
  Then, setting $\cM=[\cH_{n}\cH_{n}^{\top}]^{\inv}$ and $\bM=\cM\otimes \bI_{d}$, we have
  \begin{align*}
    \norm{\Cmp}_{\bM}^{2}
    &\le\norm{\cmpbar}^{2}_{2}+\frac{1}{4}\sum_{i=0}^{\log_{2}(T)}\localP(\vec{\cmp},2^{i})\le \norm{\cmpbar}^{2}_{2}+\frac{1}{4}\Log{T}\max_{\tau}\localP(\vec{\cmp},\tau),\\
    \norm{\Gt}^{2}_{\bM^{\inv}}
    &=
      \norm{\gt}^{2}_{2}(1+\log T)~.
  \end{align*}
\end{restatable}
Summarizing, by applying \Cref{alg:redux} with $\cM=[\cH_{n}\cH_{n}^{\top}]^\inv$ we ensure
regret
\begin{align*}
  R_{T}(\vec{\cmp})
  &\le
    \tilde\scO\brac{\sqrt{\brac{\norm{\cmpbar}^{2}_{2}+\max_{\tau}\sum_{i=1}^{N_{\tau}/2}\norm{\cmpbar_{2i+1}^{(\tau)}-\cmpbar_{2i}^{(\tau)}}^{2}_{2}}\sumtT \norm{\gt}_2^{2}}}\ .
\end{align*}
This is the first \emph{fully decoupled} guarantee for general dynamic regret
which incurs no pessimistic multiplicative penalties of the form $\max_{t,t'}\norm{\cmp_{t}-\cmp_{t'}}$. That is, the terms depending on the comparators and the terms depending on the gradients appear in separate sums.
Moreover,
observe that this measure of variability
can immediately be related to the
more standard (first-order/non-squared) path-length using the local averaging lemma of
\citet{zhang2023unconstrained} (Lemma D.7). We have
\begin{align*}
  \norm{\Cmp}^{2}_{\bM}
  &\le
    \norm{\cmpbar}^{2}_{2}+\frac{\log_{2}T}{4}\max_{\tau}\sum_{i=1}^{N_{\tau}/2}\norm{\cmpbar^{(\tau)}_{2i-1}-\cmpbar^{(\tau)}_{2^{i}}}^{2}_{2}
  \le
    \tilde \scO\brac{\bar{D}^{2}+\max_{\tau}\bar{D}\sum_{i=1}^{N_{\tau}/2}\norm{\cmpbar^{(\tau)}_{2i-1}-\cmpbar^{(\tau)}_{2^{i}}}_{2}}\\
  &\le
    \tilde \scO\brac{\bar{D}^{2}+\bar{D}\sumtTmm\norm{\cmp_{t}-\cmp_{\tpp}}_{2}}
    \le
    \tilde \scO\brac{\bar{D}^{2}+\bar{D}P_{T}},
\end{align*}
where $\bar{D}=\max_{\tau,i}\norm{\cmpbar_{i}^{(\tau)}-\cmpbar^{(\tau)}_{i+1}}\le \max_{i,j}\norm{\cmp_{i}-\cmp_{j}}$.
Thus, applying \Cref{alg:redux} with dual-norm pair
$(\norm{\cdot}_{\cH_n^{-\top}\cH_n^{\top}}, \norm{\cdot}_{\cH_n\cH_n^{\top}})$ still
guarantees worst-case regret
\[
  R_{T}(\vec{\cmp})
  \le
    \tilde \scO\brac{\norm{\Cmp}_{\cH_n^{-\top}\cH_n^{\inv}}\sqrt{{\textstyle\sumtT} \norm{\Gt}^{2}_{\cH_n\cH_n^{\top}}}}
    \le
    \tilde \scO\brac{\sqrt{\left(\norm{\bar{\cmp}}^{2}_{2}+\bar{D} P_{T}\right){\textstyle\sumtT} \norm{\gt}_2^{2}}},
\]
which matches the guarantees of prior works, up to
polylogarithmic terms.

Importantly, with $\bM=\bH_{n}^{-\top}\bH_{n}^{\inv}\otimes\bI_{d}$ the dual-norm pair
$(\norm{\cdot}_{\bM}, \norm{\cdot}_{\bM^{\inv}})$
leads to updates that can be implemented efficiently,
in requiring only $O(\log T )$ variables to be updated on each round.
This is because the Haar basis matrices are \emph{locally supported} ---
the columns of $\bH_{n}=\pmat{\bh^{(1)}&\dots&\bh^{(T)}}\in\R^{T\times T}$
form an orthogonal basis with the property that
for any $t$, $[\bh^{(i)}]_{t}\ne 0$ for only $1+\log_{2}T$ indices $i$ (see \Cref{prop:sparse-support}).
Hence,
\(
  (\bH^\top\otimes\bI_d)\Gt = (\bH^{\top}\otimes I_{d}) (\et\otimes \gt) = (\bH^{\top}\et)\otimes\gt,
\)
is a block vector with only $1+\log_{2}T$ active blocks,
requiring that we
update only $O(d\log T)$ indices to
maintain each of the variables needed to implement \Cref{alg:1d-redux}.
We provide the full details of this
computation in \Cref{app:haar-cmput}, which
we summarize below in \Cref{prop:haar-cmput}.
\begin{restatable}{proposition}{HaarCmput}\label{prop:haar-cmput}
  The algorithm characterized by applying \Cref{thm:simple-dynamic} with
  $\bS = [\bH_{n}\bH_{n}^{\top}]^{\inv}$ can be implemented with
  $\scO\brac{d\log T}$ per-round computation.
\end{restatable}

  %auto-ignore

\section{\SecCoupling}%
\label{sec:coupling}

Our main focus throughout the paper has been on
designing
algorithms
that achieve a regret bounds of the form
$R_{T}(\vec{\cmp})\le O\brac{\sqrt{f(\cmp_{1},\ldots,\cmp_{T})V(\g_{1},\ldots,\g_{T})}}$
for some functions $f$ and $V$, which cleanly separates the penalties
associated with difficult \emph{loss} sequences from the penalties associated
with difficult \emph{comparator} sequences.
However, the first works to achieve unconstrained dynamic regret
guarantees uncovered guarantees of a slightly different form,
containing a \emph{gradient-comparator correlation} penalty:
\begin{align}
  R_{T}(\vec{\cmp})\le \tilde O\brac{\sqrt{\sumtTmm\norm{\cmp_{t}-\cmp_{\tpp}}\smash{\underbrace{\sumtT \norm{\gt}^{2}\norm{\cmp_{t}-\bar{\cmp}} }_{\text{Variance/Variability
  coupling}}}}},\label{eq:coupled-regret}\\\nonumber
\end{align}
for some reference point $\bar{\cmp}$
\citep{jacobsen2022parameter,zhang2023unconstrained}. Guarantees of this
form allow some degree of coupling between the variability and variance
penalties. This can be appealing in certain situations. For instance, guarantees
of the form above have the appealing property that
the variance penalty completely disappears on any rounds where the
comparator $\cmp_{t}$ matches the reference point $\bar{\cmp}$. This can be
a very powerful property when one has \emph{a priori} access to
a benchmark model (represented by $\bar{\cmp}$) which can be expected to
predict well \emph{on average}, so that we accumulate the variance penalties
only when facing atypical/unexpected conditions.

The prior works achieving a coupling guarantee do so using rather
mysterious means. For instance, the guarantee of \citet{jacobsen2022parameter}
achieves the coupling guarantee almost by coincidence, as it appears in
response to a composite regularizer they add to the update to cancel out certain
unstable terms in the analysis, and the analysis of
\citet{zhang2023unconstrained} recovers a guarantee of a similar form
using a rather difficult analysis of the frequency-domain
representation of $\Cmp$ after projecting onto the Haar basis vectors. So far there is no unifying explanation of the
principles leading to these sorts of guarantees.

Our equivalence in~\Cref{prop:dynamic-to-static} instead shows that guarantees of the form \Cref{eq:coupled-regret}
can instead be understood through the lens of reward-regret duality, a
standard
tool used to design algorithms in the static regret setting. The reward-regret
duality states
that in order to guarantee regret of the form $R_{T}(\cmp)\le f(\cmp)$ for all
$\cmp\in\cW$,
it suffices to design an algorithm that guarantees
$-\sumtT \inner{\gt,\wt}\ge f^{*}(-\sumtT \gt)$ for any $\g_{1},\ldots,\g_{T}$.
Using \Cref{prop:dynamic-to-static}, we immediately have the following analogous
design principle for dynamic regret. Proof is deferred to \Cref{app:sequence-reward-regret}.
\begin{restatable}{theorem}{SequenceRewardRegret}\label{thm:sequence-reward-regret}
  Let $\Wealth_{T}:=-\sumtT\inner{\Gt,\Wt}$ denote the ``wealth'' of an algorithm
  $\cA$ and let $(f,f^{*})$ be a Fenchel conjugate pair.
  Then $\cA$ guarantees
  $\Wealth_{T}\ge f_{T}^{*}\big(-\sumtT \Gt\big)$ for any sequence
  $\G_{1},\ldots,\G_{T}$ if and only if $R_{T}(\vec{\cmp})\le f_{T}(\Cmp)$
  for any sequence $\vec{\cmp}=(\cmp_{1},\ldots,\cmp_{T})$ in $\cW$,
  where $\Cmp=(\cmp_{1}^{\top},\dots,\cmp_{T}^{\top})^{\top}$ is the concatenation of
  the sequence $\vec{\cmp}$
  into a vector.
\end{restatable}
So, suppose we would like to design an algorithm that guarantees for any sequence
$\vec{\cmp}=(\cmp_{1},\ldots,\cmp_{T})$ and any $\vec{\g}=(\g_{1},\ldots,\g_{T})$
regret of the form
\begin{align*}
  R_{T}(\vec{\cmp})
  &\le \sqrt{f_{T}(\Cmp)V_{T}(\Cmp)},
\end{align*}
for some $f_{T}(\Cmp)$ and $V_{T}(\Cmp)=V_{T}(\Cmp;\vec{\g})$.
Then, since
$\sqrt{ab}=\min_{\eta\ge 0} \frac{a}{2\eta}+\frac{\eta}{2}b$,
any such algorithm must have
$R_{T}(\vec{\cmp})\le \frac{f_{T}(\Cmp)}{2\eta}+\frac{\eta}{2}V_{T}(\Cmp)$
for every $\eta\ge 0$.
So, via \Cref{prop:dynamic-to-static} and the the reward-regret duality of
\Cref{thm:sequence-reward-regret}, we have that the desired guarantee is equivalent to guaranteeing for all $\eta\ge 0$ a wealth lower bound of
\[
  \Wealth_{t}
  =
    -\sumtT \inner{\Gt,\Wt} \ge \sbrac{ \frac{f_{T}(\cdot)}{2\eta}+\frac{\eta}{2}V_{T}(\cdot) }^{*}\big(-\G_{1:T}\big)
  =
    \frac{f_{T}^{*}\big(-2\eta\G_{1:T}\big)}{2\eta}\ \square \ 2\eta V^{*}_{T}\brac{ \frac{\G_{1:T}}{2\eta}},
\]
where $f^{*}_{T}$ and $V^{*}_{T}$ are the Fenchel conjugates of $f_{T}$ and $V_{T}$
respectively, and $(f_{1}\ \square\ f_{2})$ denotes the \emph{infimal convolution}~
\citep{Rockafellar70,hiriart2004fundamentals} of
$f_{1}$ and $f_{2}$:
\begin{align*}
  (f_{1}\ \square\ f_{2})(z) = \inf\Set{f_{1}(y)+f_{2}(z-y)}.
\end{align*}
Thus, the variance/variability coupling
guarantees observed in \Cref{eq:coupled-regret} can be interpreted as
achieving wealth lower-bounds for potential functions involving
infimal convolution.

The above discussion provides a general characterization of
variance/variability coupling guarantees,
though it is admittedly less clear how difficult it is
to design algorithms from this perspective due to the rather complicated
potential function that appears.
Nonetheless, we believe that this provides a valuable
perspective and insight that could be of general interest.
An important direction for
future work is to develop useful tools for working with potential
functions of this form.

%auto-ignore
\section{Conclusion}%
\label{sec:conclusion}

In this paper, we have shown a way to reduce the problem of dynamic regret minimization to the static one.
We proved a novel frontier of lower bounds showing a fundamental trade-off
between penalties on the comparators and penalties on the variance of the gradients.
In particular, we have shown that it is not possible to achieve a
guarantee that scales with
$\sqrt{\sumtTmm\norm{\cmp_{t}-\cmp_{\tpp}}^{2}}$
without incurring a variance penalty of $\scO(GT)$.
We developed a simple framework for achieving guarantees along the
lower bound frontier, and used it to develop the first algorithm
making a non-trivial variance/variability decoupling guarantee against arbitrary comparator sequences.
Our framework is simple but powerful, allowing one to fully utilize
the rich literature of static regret algorithms for online learning.

We conclude by noting some directions for future work.
There is a lot of exciting
potential to explore different measures of variability induced by different
choices of the matrix $\bM$, as well as going beyond weighted norms.
As mentioned in
\Cref{sec:coupling}, developing a useful toolset for
potential functions involving  infimal convolution
is an important next-step for developing and understanding
guarantees with a coupled variance/variability penalty, such as \Cref{eq:coupled-regret}.
Also, our lower bound in \Cref{sec:lb}
illustrates the variance-variability trade-off, but
achieving the correct logarithmic dependencies
proved to be very challenging --- many of the standard tools for
proving lower bounds in unconstrained settings revolve around anti-concentration
results that do not readily extend to arbitrary weighted norms and higher-dimensions.
We look forward to exciting development in these future directions.

\section*{Acknowledgments}
We thank Yu-Xiang Wang for the discussion on the function classes studied in non-parametric regression theory.

\bibliographystyle{plainnat_nourl}
\bibliography{paper}

\begin{thebibliography}{51}
\providecommand{\natexlab}[1]{#1}
\providecommand{\url}[1]{\texttt{#1}}
\expandafter\ifx\csname urlstyle\endcsname\relax
  \providecommand{\doi}[1]{doi: #1}\else
  \providecommand{\doi}{doi: \begingroup \urlstyle{rm}\Url}\fi

\bibitem[Baby and Wang(2019)]{BabyW19}
Dheeraj Baby and Yu-Xiang Wang.
\newblock Online forecasting of total-variation-bounded sequences.
\newblock In \emph{Advances in Neural Information Processing Systems},
  volume~32, 2019.

\bibitem[Baby and Wang(2021)]{baby2021optimal}
Dheeraj Baby and Yu-Xiang Wang.
\newblock Optimal dynamic regret in exp-concave online learning.
\newblock In \emph{Conference on Learning Theory}, pages 359--409. PMLR, 2021.

\bibitem[Besbes et~al.(2015)Besbes, Gur, and Zeevi]{besbes2015nonstationary}
Omar Besbes, Yonatan Gur, and Assaf Zeevi.
\newblock Non-stationary stochastic optimization.
\newblock \emph{Operations Research}, 63\penalty0 (5):\penalty0 1227--1244,
  2015.
\newblock \doi{10.1287/opre.2015.1408}.

\bibitem[Campolongo and Orabona(2021)]{campolongo2021closer}
Nicolò Campolongo and Francesco Orabona.
\newblock A closer look at temporal variability in dynamic online learning,
  2021.

\bibitem[Cesa-Bianchi and Lugosi(2006)]{Cesa-BianchiL06}
Nicol\`{o} Cesa-Bianchi and Gabor Lugosi.
\newblock \emph{Prediction, learning, and games}.
\newblock Cambridge University Press, 2006.

\bibitem[Cesa-Bianchi and Orabona(2021)]{Cesa-BianchiO21}
Nicol\`{o} Cesa-Bianchi and Francesco Orabona.
\newblock Online learning algorithms.
\newblock \emph{Annual Review of Statistics and Its Application}, 8:\penalty0
  165--190, 2021.

\bibitem[Chang and Shahrampour(2021)]{chang2021online}
Ting-Jui Chang and Shahin Shahrampour.
\newblock On online optimization: Dynamic regret analysis of strongly convex
  and smooth problems.
\newblock In \emph{Proceedings of the AAAI Conference on Artificial
  Intelligence}, pages 6966--6973, 2021.

\bibitem[Chen et~al.(2022)Chen, Cutkosky, and Orabona]{ChenCO22}
Keyi Chen, Ashok Cutkosky, and Francesco Orabona.
\newblock Implicit parameter-free online learning with truncated linear models.
\newblock In \emph{International Conference on Algorithmic Learning Theory},
  pages 148--175. PMLR, 2022.

\bibitem[Cutkosky and Orabona(2018)]{cutkosky2018black}
Ashok Cutkosky and Francesco Orabona.
\newblock Black-box reductions for parameter-free online learning in banach
  spaces.
\newblock In Sébastien Bubeck, Vianney Perchet, and Philippe Rigollet,
  editors, \emph{Proceedings of the 31st Conference On Learning Theory},
  volume~75 of \emph{Proceedings of Machine Learning Research}, pages
  1493--1529. PMLR, 06--09 Jul 2018.

\bibitem[Dinur et~al.(2006)Dinur, Friedgut, Kindler, and
  O'Donnell]{dinur2006fourier}
Irit Dinur, Ehud Friedgut, Guy Kindler, and Ryan O'Donnell.
\newblock On the {Fourier} tails of bounded functions over the discrete cube.
\newblock In \emph{Proceedings of the thirty-eighth annual ACM symposium on
  Theory of computing}, pages 437--446, 2006.

\bibitem[Falkowski(1998)]{falkowski1998generalized}
Bogdan~J Falkowski.
\newblock Generalized haar spectral representations and their applications.
\newblock \emph{Nanyang Technological University. Singapore}, 1998.

\bibitem[Foster et~al.(2018)Foster, Rakhlin, and Sridharan]{foster2018online}
Dylan~J. Foster, Alexander Rakhlin, and Karthik Sridharan.
\newblock Online learning: Sufficient statistics and the burkholder method.
\newblock In Sébastien Bubeck, Vianney Perchet, and Philippe Rigollet,
  editors, \emph{Proceedings of the 31st Conference On Learning Theory},
  volume~75 of \emph{Proceedings of Machine Learning Research}, pages
  3028--3064. PMLR, 06--09 Jul 2018.

\bibitem[Golub and Van~Loan(2013)]{golub2013matrix}
Gene~H Golub and Charles~F Van~Loan.
\newblock \emph{Matrix computations}.
\newblock JHU press, 2013.

\bibitem[Gordon(1999)]{Gordon99b}
Geoffrey~J. Gordon.
\newblock Regret bounds for prediction problems.
\newblock In \emph{Proc. of the twelfth annual conference on Computational
  learning theory (COLT)}, pages 29--40, 1999.

\bibitem[Gyorgy and Szepesvari(2016)]{gyorgy2016shifting}
Andras Gyorgy and Csaba Szepesvari.
\newblock Shifting regret, mirror descent, and matrices.
\newblock In Maria~Florina Balcan and Kilian~Q. Weinberger, editors,
  \emph{Proceedings of The 33rd International Conference on Machine Learning},
  volume~48 of \emph{Proceedings of Machine Learning Research}, pages
  2943--2951, New York, New York, USA, 20--22 Jun 2016. PMLR.

\bibitem[Hall and Willett(2016)]{hall2016online}
Eric~C. Hall and Rebecca~M. Willett.
\newblock Online optimization in dynamic environments, 2016.

\bibitem[Herbster and Warmuth(1998)]{HerbsterW98b}
Mark Herbster and Manfred~K Warmuth.
\newblock Tracking the best regressor.
\newblock In \emph{Proceedings of the eleventh annual conference on
  Computational learning theory}, pages 24--31, 1998.

\bibitem[Herbster and Warmuth(2001)]{HerbsterW01}
Mark Herbster and Manfred~K Warmuth.
\newblock Tracking the best linear predictor.
\newblock \emph{Journal of Machine Learning Research}, 1\penalty0
  (281-309):\penalty0 10--1162, 2001.

\bibitem[Hiriart-Urruty and Lemar{\'e}chal(2004)]{hiriart2004fundamentals}
Jean-Baptiste Hiriart-Urruty and Claude Lemar{\'e}chal.
\newblock \emph{Fundamentals of convex analysis}.
\newblock Springer Science \& Business Media, 2004.

\bibitem[Jacobsen and Cutkosky(2022)]{jacobsen2022parameter}
Andrew Jacobsen and Ashok Cutkosky.
\newblock Parameter-free mirror descent.
\newblock In Po-Ling Loh and Maxim Raginsky, editors, \emph{Proceedings of
  Thirty Fifth Conference on Learning Theory}, volume 178 of \emph{Proceedings
  of Machine Learning Research}, pages 4160--4211. PMLR, 02--05 Jul 2022.

\bibitem[Jacobsen and Cutkosky(2023)]{jacobsen2023unconstrained}
Andrew Jacobsen and Ashok Cutkosky.
\newblock Unconstrained online learning with unbounded losses.
\newblock In \emph{International Conference on Machine Learning (ICML)}. PMLR,
  2023.

\bibitem[Jadbabaie et~al.(2015)Jadbabaie, Rakhlin, Shahrampour, and
  Sridharan]{jadbabaie2015online}
Ali Jadbabaie, Alexander Rakhlin, Shahin Shahrampour, and Karthik Sridharan.
\newblock {Online Optimization : Competing with Dynamic Comparators}.
\newblock In Guy Lebanon and S.~V.~N. Vishwanathan, editors, \emph{Proceedings
  of the Eighteenth International Conference on Artificial Intelligence and
  Statistics}, volume~38 of \emph{Proceedings of Machine Learning Research},
  pages 398--406, San Diego, California, USA, 09--12 May 2015. PMLR.

\bibitem[Johnson(1970)]{johnson1970positive}
Charles~Royal Johnson.
\newblock Positive definite matrices.
\newblock \emph{The American Mathematical Monthly}, 77\penalty0 (3):\penalty0
  259--264, 1970.

\bibitem[Koolen et~al.(2015)Koolen, Malek, Bartlett, and
  Abbasi-Yadkori]{KoolenMBAY15}
Wouter~M Koolen, Alan Malek, Peter~L Bartlett, and Yasin Abbasi-Yadkori.
\newblock Minimax time series prediction.
\newblock In \emph{Advances in Neural Information Processing Systems},
  volume~28, 2015.

\bibitem[Luo et~al.(2022)Luo, Zhang, Zhao, and Zhou]{luo2022corralling}
Haipeng Luo, Mengxiao Zhang, Peng Zhao, and Zhi-Hua Zhou.
\newblock Corralling a larger band of bandits: A case study on switching regret
  for linear bandits.
\newblock In Po-Ling Loh and Maxim Raginsky, editors, \emph{Proceedings of
  Thirty Fifth Conference on Learning Theory}, volume 178 of \emph{Proceedings
  of Machine Learning Research}, pages 3635--3684. PMLR, 02--05 Jul 2022.

\bibitem[Mcmahan and Streeter(2012)]{mcmahan2012noregret}
Brendan Mcmahan and Matthew Streeter.
\newblock No-regret algorithms for unconstrained online convex optimization.
\newblock In F.~Pereira, C.~J.~C. Burges, L.~Bottou, and K.~Q. Weinberger,
  editors, \emph{Advances in Neural Information Processing Systems}, volume~25.
  Curran Associates, Inc., 2012.

\bibitem[McMahan and Orabona(2014)]{mcmahan2014unconstrained}
H.~Brendan McMahan and Francesco Orabona.
\newblock Unconstrained online linear learning in {Hilbert} spaces: Minimax
  algorithms and normal approximations.
\newblock In Maria~Florina Balcan, Vitaly Feldman, and Csaba Szepesvári,
  editors, \emph{Proceedings of The 27th Conference on Learning Theory},
  volume~35 of \emph{Proceedings of Machine Learning Research}, pages
  1020--1039, Barcelona, Spain, 13--15 Jun 2014. PMLR.

\bibitem[Mhammedi and Koolen(2020)]{MhammediK20}
Zakaria Mhammedi and Wouter~M. Koolen.
\newblock Lipschitz and comparator-norm adaptivity in online learning.
\newblock In \emph{Conference on Learning Theory}, pages 2858--2887. PMLR,
  2020.

\bibitem[Nemirovskij and Yudin(1983)]{NemirovskijY83}
Arkadij~Semenovi\v{c} Nemirovskij and David~Borisovich Yudin.
\newblock \emph{Problem complexity and method efficiency in optimization}.
\newblock Wiley, New York, NY, USA, 1983.

\bibitem[O'Donnell and Zhao(2015)]{o2015polynomial}
Ryan O'Donnell and Yu~Zhao.
\newblock Polynomial bounds for decoupling, with applications.
\newblock \emph{arXiv preprint arXiv:1512.01603}, 2015.

\bibitem[Orabona(2019)]{Orabona19}
Francesco Orabona.
\newblock A modern introduction to online learning.
\newblock \emph{arXiv preprint arXiv:1912.13213}, 2019.
\newblock Version 6.

\bibitem[Orabona and P\'{a}l(2016)]{orabona2016coin}
Francesco Orabona and D\'{a}vid P\'{a}l.
\newblock Coin betting and parameter-free online learning.
\newblock In \emph{Proceedings of the 30th International Conference on Neural
  Information Processing Systems}, NIPS'16, page 577–585, Red Hook, NY, USA,
  2016. Curran Associates Inc.

\bibitem[Orabona and Pál(2018)]{orabona2018scale}
Francesco Orabona and Dávid Pál.
\newblock Scale-free online learning.
\newblock \emph{Theoretical Computer Science}, 716:\penalty0 50 -- 69, 2018.
\newblock ISSN 0304-3975.
\newblock \doi{https://doi.org/10.1016/j.tcs.2017.11.021}.
\newblock Special Issue on ALT 2015.

\bibitem[Rockafellar(1970)]{Rockafellar70}
R.~Tyrrell Rockafellar.
\newblock \emph{Convex Analysis}.
\newblock Princeton University Press, 1970.

\bibitem[Sadhanala et~al.(2016)Sadhanala, Wang, and Tibshirani]{SadhanalaWT16}
Veeranjaneyulu Sadhanala, Yu-Xiang Wang, and Ryan~J Tibshirani.
\newblock Total variation classes beyond 1d: Minimax rates, and the limitations
  of linear smoothers.
\newblock In \emph{Advances in Neural Information Processing Systems},
  volume~29, 2016.

\bibitem[Shalev-Shwartz(2011)]{shalev2011online}
Shai Shalev-Shwartz.
\newblock Online learning and online convex optimization.
\newblock \emph{Foundations and Trends in Machine Learning}, 4\penalty0 (2),
  2011.

\bibitem[Stanković and Falkowski(2003)]{stankovic2003haar}
Radomir~S. Stanković and Bogdan~J. Falkowski.
\newblock The {Haar} wavelet transform: its status and achievements.
\newblock \emph{Computers \& Electrical Engineering}, 29\penalty0 (1):\penalty0
  25--44, 2003.
\newblock ISSN 0045-7906.

\bibitem[Steeb and Shi(1997)]{steeb1997matrix}
Willi-Hans Steeb and Tan~Kiat Shi.
\newblock \emph{Matrix calculus and Kronecker product with applications and C++
  programs}.
\newblock World Scientific, 1997.

\bibitem[Stoer et~al.(1980)Stoer, Bulirsch, Bartels, Gautschi, and
  Witzgall]{stoer1980introduction}
Josef Stoer, Roland Bulirsch, R~Bartels, Walter Gautschi, and Christoph
  Witzgall.
\newblock \emph{Introduction to numerical analysis}, volume~2.
\newblock Springer, 1980.

\bibitem[Streeter and McMahan(2010)]{StreeterM10}
Matthew Streeter and H~Brendan McMahan.
\newblock Less regret via online conditioning.
\newblock \emph{arXiv preprint arXiv:1002.4862}, 2010.

\bibitem[Walnut(2013)]{walnut2013introduction}
David~F Walnut.
\newblock \emph{An introduction to wavelet analysis}.
\newblock Springer Science \& Business Media, 2013.

\bibitem[Warmuth and Jagota(1997)]{WarmuthJ97}
Manfred~K Warmuth and Arun~K Jagota.
\newblock Continuous and discrete-time nonlinear gradient descent: Relative
  loss bounds and convergence.
\newblock In \emph{Electronic proceedings of the 5th International Symposium on
  Artificial Intelligence and Mathematics}, 1997.

\bibitem[Wolkowicz and Styan(1980)]{wolkowicz1980bounds}
Henry Wolkowicz and George~PH Styan.
\newblock Bounds for eigenvalues using traces.
\newblock \emph{Linear algebra and its applications}, 29:\penalty0 471--506,
  1980.

\bibitem[Yang et~al.(2016)Yang, Zhang, Jin, and Yi]{yang2016tracking}
Tianbao Yang, Lijun Zhang, Rong Jin, and Jinfeng Yi.
\newblock Tracking slowly moving clairvoyant: Optimal dynamic regret of online
  learning with true and noisy gradient.
\newblock In Maria~Florina Balcan and Kilian~Q. Weinberger, editors,
  \emph{Proceedings of The 33rd International Conference on Machine Learning},
  volume~48 of \emph{Proceedings of Machine Learning Research}, pages 449--457,
  New York, New York, USA, 2016. PMLR.

\bibitem[Zhang et~al.(2017)Zhang, Yang, Yi, Jin, and Zhou]{zhang2017improved}
Lijun Zhang, Tianbao Yang, Jinfeng Yi, Rong Jin, and Zhi-Hua Zhou.
\newblock Improved dynamic regret for non-degenerate functions.
\newblock In I.~Guyon, U.~V. Luxburg, S.~Bengio, H.~Wallach, R.~Fergus,
  S.~Vishwanathan, and R.~Garnett, editors, \emph{Advances in Neural
  Information Processing Systems}, volume~30. Curran Associates, Inc., 2017.

\bibitem[Zhang et~al.(2018)Zhang, Lu, and Zhou]{zhang2018adaptive}
Lijun Zhang, Shiyin Lu, and Zhi-Hua Zhou.
\newblock Adaptive online learning in dynamic environments.
\newblock In \emph{Proceedings of the 32nd International Conference on Neural
  Information Processing Systems}, pages 1330--1340, 2018.

\bibitem[Zhang et~al.(2024{\natexlab{a}})Zhang, Cutkosky, and
  Paschalidis]{zhang2023unconstrained}
Zhiyu Zhang, Ashok Cutkosky, and Yannis Paschalidis.
\newblock Unconstrained dynamic regret via sparse coding.
\newblock \emph{Advances in Neural Information Processing Systems}, 36,
  2024{\natexlab{a}}.

\bibitem[Zhang et~al.(2024{\natexlab{b}})Zhang, Yang, Cutkosky, and
  Paschalidis]{zhang2023improving}
Zhiyu Zhang, Heng Yang, Ashok Cutkosky, and Ioannis~C Paschalidis.
\newblock Improving adaptive online learning using refined discretization.
\newblock In \emph{International Conference on Algorithmic Learning Theory},
  pages 1208--1233. PMLR, 2024{\natexlab{b}}.

\bibitem[Zhao et~al.(2022)Zhao, Xie, Zhang, and Zhou]{zhao2022efficient}
Peng Zhao, Yan-Feng Xie, Lijun Zhang, and Zhi-Hua Zhou.
\newblock Efficient methods for non-stationary online learning.
\newblock In S.~Koyejo, S.~Mohamed, A.~Agarwal, D.~Belgrave, K.~Cho, and A.~Oh,
  editors, \emph{Advances in Neural Information Processing Systems}, volume~35,
  pages 11573--11585. Curran Associates, Inc., 2022.

\bibitem[Zhao et~al.(2024)Zhao, Zhang, Zhang, and Zhou]{zhao2024adaptivity}
Peng Zhao, Yu-Jie Zhang, Lijun Zhang, and Zhi-Hua Zhou.
\newblock Adaptivity and non-stationarity: Problem-dependent dynamic regret for
  online convex optimization.
\newblock \emph{Journal of Machine Learning Research}, 25\penalty0
  (98):\penalty0 1--52, 2024.

\bibitem[Zinkevich(2003)]{zinkevich2003online}
Martin Zinkevich.
\newblock Online convex programming and generalized infinitesimal gradient
  ascent.
\newblock In \emph{Proceedings of the 20th international conference on machine
  learning (icml-03)}, pages 928--936, 2003.

\end{thebibliography}

%auto-ignore

\clearpage
\appendix

%auto-ignore
\section{Proofs for Section~\ref{sec:lb} (\SecLB)}

In this section, we provide proof of our main lower bound result
from \Cref{sec:lb}.
We first introduce a technical tool from the
literature on decoupling theory and a key lemma (\Cref{lemma:wealth-bound}).
Proof of our main result is in \Cref{app:pf-lb}.

Consider a function $f:[-1,1]^d \to \R$, defined as
\[
f(\bx) = \sum_{i,j} A_{i,j} x_i x_j~,
\]
Define $\bA$ the matrix with elements $A_{i,j}$.
In this section we will use the following notations
for quantities related to a polynomial induced
by the quadratic form $\bx\mapsto\inner{\bx,\bA\bx}$
(see page 6 of \citet{o2015polynomial})
\begin{align*}
\Var[f]=\sum_{i,j} A_{i,j}^2 = \|A\|_F^2,\\
\Inf_i[f] = \sum_{j=1}^d (A_{i,j}^2 + A_{j,i}^2)~.
\end{align*}
One of the key difficulties in deriving the lower bound is that squared weighted
norms $\bx\mapsto\inner{\bx,\bA\bx}$ introduce dependencies between the
coordinates of $\bx$, which breaks the usual lower bound
arguments which rely on anti-concentration of \emph{independent}
Rademacher random variables. Instead, we must leverage an anti-concentration
result that holds for \emph{polynomials} of random variables.
\begin{theorem}[Theorem 3 of \citet{dinur2006fourier}]
  \label{thm:chaos-lb}
  There is a universal constant $C$ such that the following holds. Suppose
  $G:\Set{\pm1}^{d}\to \R$ is a polynomial of degree at most $2$ and assume
  $\Var[g]=1$. Let $t\ge1$ and suppose that $\Inf_{i}[g]\le C^{-2}t^{-2}$ for all
  $i\in[d]$. Then
  \begin{align*}
    \Prob{\abs{g(x)}\ge t}\ge \Exp{-C^{2}t^{2}4\log{2}}\ .
  \end{align*}
\end{theorem}

Using this anti-concentration result, the following key lemma provides
a general lower bound on the wealth obtainable by any algorithm, subject to
the weighting imposed by a matrix $\bA$.
\begin{restatable}{lemma}{WealthBound}\label{lemma:wealth-bound}
  Let $\cA$ be an online learning algorithm, and suppose $\cA$ guarantees $R_{T}(0)\le G\epsilon_{T}$ for any sequence of linear losses
  $\sg_{1},\dots,\sg_{T}\in\R$ satisfying $\abs{\sgt}\le G$.
  Let $\bA\in\R^{T\times T}$  be any symmetric positive definite matrix,
  and let $\bB=\bA-\Diag{\bA}$.
  Then, there is a universal constant $C>0$ such that for any
  $1\le q\le \frac{\norm{\bB}_{F}}{C\sqrt{2\max_{i}\sum_{j=1}^{T} B_{ij}^{2}}}$, there is a sequence of losses
  $g_{1},\ldots,g_{T}\in\R$ such that
  \begin{align*}
    \norm{\pmat{g_{1}\\\vdots\\g_{T}}}_{\bA}^{2}\ge G^{2}\sbrac{\Tr(\bA)+q\norm{\bA-\Diag{\bA}}_{F}}
  \end{align*}
  and
  \begin{align*}
    R_{T}(0)\ge G\epsilon_{T}\sbrac{1-2^{4C^{2}q^{2}}}\ .
  \end{align*}
\end{restatable}
\begin{proof}
  Let $Y_{1},\ldots,Y_{T}$ be independent Rademacher random variables
  and
  set $\sg_{t}= G\,Y_{t}$, so that $\EE{R_{T}(0)}= \EE{\sumtT \sgt\swt}=0$.
  Then, using the regret equivalence of \Cref{prop:dynamic-to-static}
  and conditioning on any event
  $\cE$ with $\Prob{\cE}>0$, we have
  \begin{align*}
    0
    &=
      \EE{R_{T}(0)}\\
    &=
      \EE{R_{T}(0)\Big|\cE}\Prob{\cE} + \EE{R_{T}(0)\Big|\cE^{c}}\Prob{\cE^{c}}\\
    &\le
      \EE{R_{T}(0)|\cE}\Prob{\cE} + G\,\epsilon_{T}\brac{1-\Prob{\cE}},
  \end{align*}
  where the last line uses the fact that $\cA$ guarantees
  $R_{T}(0)\le G\epsilon_{T}$ for any $\sg_{1},\ldots,\sg_{T}$ satisfying
  $\abs{\sgt}\le G$ for  all $t$. Re-arranging, we have
  \begin{align}
    \EE{R_{T}(0)\Big|\cE}\ge G\epsilon_{T}\brac{1-\frac{1}{\Prob{\cE}}}~.\label{eq:expected-wealth}
  \end{align}
  Next, let $\Gt=\et\otimes g_{t}$ for all $t$
  and consider the event
  \[\cE=\Set{\norm{\sumtT \Gt}^{2}_{\bA}=\norm{(g_{1},\dots,g_{T})^{\top}}_{\bA}^{2}\ge \Tr(\bA)+q\norm{\bA-\Diag{\bA}}_{F}}\]
  for some $q>0$.
  We proceed by lower bounding the probability of this event.

  Observe that
  \begin{align*}
    \norm{\sumtT\Gt}^{2}_{\bA} = G^{2}\sum_{i,j} Y_{i}Y_{j}A_{ij} = G^{2}\sbrac{\Tr(\bA)+\sum_{i,j\ne i}Y_{i}Y_{j}A_{ij}}.
  \end{align*}
  Denote $\bB=\bA-\Diag{\bA}$ and
  note that $f(Y_{1},\dots,Y_{T})=\sum_{i,j}Y_{i}Y_{j} B_{ij}$
  is a polynomial of degree at most $2$
  and variance $\Var[f]=\sum_{i,j} B_{ij}^{2}=\norm{\bA-\Diag{\bA}}^{2}_{F}=\norm{\bB}^{2}_{F}$.
  Moreover, since $\bA$ is symmetric
  we have
  $\Inf_{i}[f]=\sum_{j=1}^{T}B_{ij}^{2}+B_{ji}^{2}=2\sum_{j=1}^{T}B_{ij}^{2}$ for any $i$.
  It follows that if we let $g(\bY)=\frac{f(\bY)}{\sqrt{\norm{\bB}^{2}_{F}}}=\frac{f(\bY)}{\norm{\bB}_{F}}$, then
  $g$ is a polynomial of degree at most $2$, $\Var[g]=1$, and
  for any $i\in[T]$ we have $\Inf_{i}[g]=\frac{2\sum_{j=1}^{T}B_{ij}^{2}}{\norm{\bB}^{2}_{F}}$. Hence
  by \Cref{thm:chaos-lb}, there is a universal constant $C$
  such that for any $1\le q\le \frac{\norm{\bB}_{F}}{C\sqrt{2\max_{i}\sum_{j=1}^{T}B_{ij}^{2}}}$ it holds that
  \[
    \Prob{f(\bY)\ge q\norm{\bB}_{F}}=\Prob{g(\bY)\ge q}\ge \Exp{-4C^{2}q^{2}\log{2}}
    = 2^{-4C^{2}q^{2}}~.
  \]
  Now observe that $\Prob{\cE}=\Prob{f(\bY)\ge q\norm{\bB}_{F}}$ by
  construction, so \Cref{eq:expected-wealth} can be bound as
  \begin{align*}
    \EE{R_{T}(0)\Big|\cE}
    &\ge
      G\epsilon_{T}\brac{1-\frac{1}{\Prob{\cE}}}
      =
      G\epsilon_{T}\brac{1-2^{4C^{2}q^{2}}},
  \end{align*}
  which implies the existence of a sequence $g_{1},\dots,g_{T}\in\R$
  such that $R_{T}(0)\ge G\epsilon_{T}\sbrac{1-2^{4C^{2}q^{2}}}$
  and
  \begin{align*}
    \norm{\sumtT \Gt}^{2}_{\bA}\ge G^{2}\sbrac{\Tr(\bA)+q\norm{\bB}_{F}} = G^{2}\sbrac{\Tr(\bA)+q\norm{\bA-\Diag{\bA}}_{F}},
  \end{align*}
  for any $1\le q\le \frac{\norm{\bB}_{F}}{C\sqrt{2\max_{i}\sum_{j=1}^{T}B_{ij}^{2}}}$.
\end{proof}

\subsection{Proof of Theorem~\ref{thm:pf-lb}}%
\label{app:pf-lb}

In this section we prove our main lower bound.
\PFLB*
\begin{proof}
  Denote $\bA=\bM^{\inv}$ and $\bB=\bA-\Diag{\bA}$.
  By \Cref{lemma:wealth-bound}, there is a universal constant $C$ and a sequence
  $g_{1},\ldots,g_{T}\in\R$
  such that
  for any $1\le q\le \frac{\norm{\bB}_{F}}{C\sqrt{2\max_{i}\sum_{j=1}^{T}B_{ij}^{2}}}$,
  it holds that
  \begin{align*}
    \norm{\sumtT\G_{t}}^{2}_{ \bA}\ge G^{2}\sbrac{\Tr(\bA)+q\norm{\bA-\Diag{\bA}}_{F}}
  \end{align*}
  and
  \begin{align*}
    R_{T}(0)\ge G\epsilon_{T}\sbrac{1-2^{4C^{2}q^{2}}}~.
  \end{align*}
  Hence, choosing comparator sequence $u_{1},\dots,u_{T}\in\R$ to satisfy  and
  $\Cmp=(u_{1},\dots,u_{T})^{\top}=-\sqrt{P}\frac{\bA\sumtT \Gt}{\norm{\sumtT\Gt}_{\bA}}\in\R^{T}$,
  we have $\norm{\Cmp}_{\bA^{\inv}}=\norm{\Cmp}_{\bM}= \sqrt{P}$ and
  \begin{align*}
    R_{T}(u_{1},\dots,u_{T})
    &=
      R_{T}(0)-\inner{\sumtT \Gt, \Cmp}\\
    &=
      G\sqrt{P}\norm{\sumtT \Gt}_{\bA}+R_{T}(0)\\
    &\ge
      G\sqrt{P\sbrac{\Tr(\bA)+q\norm{\bA-\Diag{\bA}}_{F}}}+R_{T}(0)\\
    &\ge
      G\epsilon_{T}+G\sqrt{P\sbrac{\Tr(\bA)+q\norm{\bA-\Diag{\bA}}_{F}}}-G\epsilon_{T}2^{4C^{2}q^{2}}~.
  \end{align*}
  Now, for $P$ satisfying
  $T_{0}:=4C^{2}\le\log_{2}\brac{\frac{\sqrt{P\sbrac{\Tr(\bA)+\norm{\bB}_{F}}}}{2\epsilon_{T}}}\le T$
  we may choose
  \begin{align*}
    q=\sqrt{\frac{\log_{2}\brac{\frac{\sqrt{P[\Tr(\bA)+\norm{\bB}_{F}]}}{2\epsilon_{T}}}}{4C^{2}}}~.
  \end{align*}
  Indeed, observe that this choice satisfies
  $1\le q\le \frac{\norm{\bB}_{F}}{C\sqrt{2\max_{i}\sum_{j=1}^{T}B_{ij}^{2}}}$
  as required:
  \begin{align*}
    1\le q=\sqrt{\frac{\log_{2}\brac{\frac{\sqrt{P[\Tr(\bA)+\norm{\bB}_{F}]}}{2\epsilon_{T}}}}{4C^{2}}}\le  \sqrt{\frac{T}{4C^{2}}}\le \frac{\norm{\bB}_{F}}{C\sqrt{2\max_{i}\sum_{j=1}^{T}B_{ij}^{2}}},
  \end{align*}
  where the final inequality uses the assumption
  $\norm{\bB}^{2}_{F}/2\max_{i}\sum_{ij}B_{ij}^{2}\ge \frac{T}{4}$.
  Hence, we have that
  \begin{align*}
    G\epsilon_{T}2^{4C^{2}q^{2}}\le \frac{G}{2}\sqrt{P\sbrac{\Tr(\bA)+\norm{\bB}_{F}}}\le \frac{G}{2}\sqrt{P\sbrac{\Tr(\bA)+q\norm{\bB}_{F}}},
  \end{align*}
  so that the overall the regret can be lower-bounded as
  \begin{align*}
    R_{T}(u_{1},\ldots,u_{T})
    &\ge
      G\epsilon_{T}+\half G\sqrt{P\sbrac{\Tr(\bA)+q\norm{\bB}_{F}}}\\
    &=
      G\epsilon_{T}+ \frac{G}{2}\sqrt{P\sbrac{\Tr(\bA)+\norm{\bB}_{F}\frac{\log^{\half}\brac{\sqrt{P[\Tr(\bA)+\norm{\bB}_{F}]}/2\epsilon_{T}}}{\sqrt{T_{0}}}}}~.
  \end{align*}
\end{proof}

\subsection{Proof of Proposition~\ref{prop:sqr-trade-off}}%
\label{app:sqr-trade-off}
\SqrTradeOff*
\begin{proof}

  We first show the properties
  that $\norm{\Cmp}^{2}_{F}=\norm{\cmp_{T}}^{2}_{2}+\sumtTmm \norm{\cmp_{t}-\cmp_{\tpp}}^{2}_{2}$ and
  $\Tr(\diffOp^{\inv}\diffOp^{-\top})=\sumtT
  \sbrac{\diffOp^{\inv}\diffOp^{-\top}}_{tt} = \sumtT T-t+1 = \frac{T(T+1)}{2}$,
  and show that $\bM$ satisfies the conditions of \Cref{thm:pf-lb} at the end.

Observe that
\begin{align*}
  (\diffOp\otimes \bI_d)\Cmp
  &=
      \begin{pmatrix}
         \bI_d& -\bI_d & \boldsymbol{0} & \boldsymbol{0} & \cdots \\
         \boldsymbol{0}& \bI_d & -\bI_d & \boldsymbol{0} & \cdots\\
         \boldsymbol{0}& \boldsymbol{0} & \bI_d & -\bI_d& \cdots\\
         \vdots&  &   &  \ddots &\\
         \boldsymbol{0} & \boldsymbol{0}& \boldsymbol{0} &\cdots &\bI_d
       \end{pmatrix}\begin{pmatrix}
                      \cmp_{1}\\
                      \vdots\\
                      \cmp_{T}
                    \end{pmatrix}
  =\begin{pmatrix}
      \cmp_{1}-\cmp_{2}\\
      \cmp_{2}-\cmp_{3}\\
      \vdots\\
      \cmp_{T-1}-\cmp_{T}\\
      \cmp_{T}
    \end{pmatrix},
\end{align*}
and since
$(\diffOp^{\top}\otimes \bI_d)(\diffOp\otimes \bI_d)=(\diffOp^{\top}\diffOp)\otimes \bI_d=\bM$,
we have
\begin{align*}
  \norm{\Cmp}^{2}_{\bM}
  &=
    \inner{\Cmp,(\diffOp^{\top}\otimes \bI_d)(\diffOp\otimes \bI_d)\Cmp}
  =
    \inner{(\diffOp\otimes \bI_d)\Cmp,(\diffOp\otimes \bI_d)\Cmp}\\
  &=
    \norm{\cmp_{T}}^{2}_{2}+\sumtTmm \norm{\cmp_{t}-\cmp_{\tpp}}^{2}_{2}\ .
\end{align*}
Using the inverse property of the Kronecker product, we also have
\begin{align*}
  \bM^{\inv}=\sbrac{\diffOp^{\top}\diffOp \otimes \bI_d}^{\inv}=\sbrac{\diffOp^{\top}\diffOp}^{\inv}\otimes \bI_d=\diffOp^{\inv}\diffOp^{-\top}\otimes \bI_d,
\end{align*}
and by \Cref{lemma:diff-op-properties} we have that $\diffOp^{\inv}$ is the
upper-triangular matrix of all $1$'s, that is,
the matrix with entries
\begin{align*}
  \diffOpScalar^{\inv}_{ij}=\begin{cases}1&\text{if }i\le j\\
                    0&\text{otherwise}\end{cases},
\end{align*}
and likewise, $\diffOp^{-\top}$ is a lower-triangular matrix of $1's$. In other
words, for any $t$ we have
\begin{align*}
  \sbrac{\diffOp^{\inv}\diffOp^{-\top}}_{tt} = \sum_{i=1}^{T}\diffOpScalar^{\inv}_{ti}\diffOpScalar^{-\top}_{it} = \sum_{i\le t}\diffOpScalar^{\inv}_{ti} = T-t+1~.
\end{align*}
So, summing over $t$ we have
\[
  \Tr(\diffOp^{\inv}\diffOp^{-\top})=\sumtT
  \sbrac{\diffOp^{\inv}\diffOp^{-\top}}_{tt} = \sumtT T-t+1 = \frac{T(T+1)}{2}~.
\]

  Now we show that $\bM$ satisfies the conditions of \Cref{thm:pf-lb}.
  $\bM=\cM\otimes \bI_d=[\diffOp^{\top}\diffOp]\otimes \bI_d$ is
  clearly symmetric since it is the Kronecker product of
  two symmetric matrices.
  Observe that
  for any
  $\bx\ne\zeros\in\R^{T}$ we have $\bSigma\bx\ne \zeros$ by positive definiteness
  of $\diffOp$ (\Cref{lemma:diff-op-properties})
  and thus $\inner{\bx,\cM\bx}=\inner{\diffOp\bx,\diffOp\bx}>0$.
  Thus,
  $\bM=\cM\otimes \bI_d$ is the Kronecker product
  of symmetric positive definite matrices, so $\bM$ is
  symmetric positive definite \citep[see, \textit{e.g.},][Chapter 2]{steeb1997matrix}.

  Lastly, let $\bB=\Sigma^{\inv}\Sigma^{-\top}-\Diag{\Sigma^{\inv}\Sigma^{-\top}}$. We are to show that
  $\norm{\bB}_{F}\ge\frac{T}{2}\sum_{j}B_{ij}^{2}$ for any $i$.
  First observe that calculation of $[\diffOp^{\inv}\diffOp^{-\top}]_{tt}$ is generalized to
  \begin{align*}
    \sbrac{\diffOp^{\inv}\diffOp^{-\top}}_{ij} = \sum_{k=1}^{T}\Sigma^{\inv}_{ik}\Sigma^{-\top}_{kj}=\sum_{k=1}^{T}\Sigma^{-\top}_{ki}\Sigma^{-\top}_{kj}= \sum_{k=1}^{j\minOp i}1 = T-\max\Set{j, i} +1,
  \end{align*}
  for any $i,j$, and likewise $\bB_{ij}=T-\max\Set{j, i}+1$ for $i\ne j$ and $0$
  otherwise, from which it is easily seen that $\max_{i}\sum_{j}B_{ij}^{2}=\sum_{j}B_{1j}^{2}$, so
  for any $i$ we have
  \begin{align*}
    \sum_{j}B_{ij}^{2}\le \sum_{j}B_{1j}^{2} = \sum_{j=2}^{T}(T-j+1)^{2} = \frac{1}{6}T(2T^{2}-3T+1)~.
  \end{align*}
  On the other hand,
  \begin{align*}
    \norm{\bB}^{2}_{F}
    &=\sum_{i}\sum_{j}B_{ij}^{2}
      =
      \frac{1}{6}T^{2}(T^{2}-1)\\
    &=
      \frac{T}{2}\frac{T}{6}(2T^{2}-2)
    =
      \frac{T}{2}\frac{T}{6}(2T^{2}-3T + 3T-2)\ge
      \frac{T}{2}\frac{T}{6}(2T^{2}-3T +1)\\
      &\ge
      \frac{T}{2}\sum_{j}B_{ij}^{2},
  \end{align*}
  for any $i$, where the last line applies the inequality in the previous display.
\end{proof}

    %auto-ignore
\subsection{Sufficiency of Proposition~\ref{prop:sqr-trade-off}}%
\label{app:lb-details}

The choice of $\bM$ in \Cref{prop:sqr-trade-off} uniquely exposes the squared
path-length up to
the constant offset term $\norm{\cmp_{T}}^{2}$.
In this section we demonstrate that the choice of offset term in
\Cref{prop:sqr-trade-off}
does not make
any significant difference for the claim that adapting to the squared
path-length
requires incurring a $\Tr(\bM^{\inv})\ge \Omega(T^{2})$ penalty, and hence
that \Cref{prop:sqr-trade-off} is sufficient to demonstrate that
adapting to the squared path-length is not possible without incurring vacuous regret.

To expedite the discussion, we first introduce two technical lemmas,
proven in \Cref{app:perturbation-bound,app:new-eigen-bound} respectively.
\begin{restatable}{lemma}{PerturbationBound}\label{lemma:perturbation-bound}
  Let $\bv\in\R^{T}$ be an arbitrary non-zero vector and let $\bB\in\R^{T\times T}$ be a
  symmetric matrix with eigenvalues $0=\lambda_{1}(\bB)<\lambda_{2}(\bB)\le \ldots\le \lambda_{T}(\bB)$.
  Then
  \begin{align*}
    \Tr((\bB + \bv\bv^{\top})^{\inv})\ge \norm{\bv}^{2}+\sum_{t=2}^{T}\frac{1}{\lambda_{t}(\bB)}~.
  \end{align*}
\end{restatable}
\begin{restatable}{lemma}{NewEigenBound}\label{lemma:new-eigen-bound}
  Let $\bSigma\in\R^{T\times T}$ denote the finite-difference matrix defined in
  \Cref{prop:sqr-trade-off} and let $\bM=\bSigma^{\top}\bSigma$.
  Then, for any $T>1$, we have
  \begin{align*}
    \lambda_{\max}(\bM^{\inv}) \le
    \frac{9}{10}\Tr(\bM^{\inv}),
  \end{align*}
  where $\lambda_{\max}(\bM^{\inv})$ is the maximal eigenvalue of $\bM^{\inv}$.
\end{restatable}

Now, Consider
the 1-dimensional setting and
note that for any positive definite $\bM$ we can find a unique $\bSigma$ such that
$\bM=\bSigma^{\top}\bSigma$. Hence,
\begin{align*}
  \norm{\Cmp}^{2}_{\bM}=\inner{\Cmp,\bM\Cmp} = \inner{\bSigma\Cmp,\bSigma\Cmp},
\end{align*}
so without loss of generality we can focus on
$\bSigma$ for which
\begin{align*}
  \langle \bSigma\Cmp, \bSigma\Cmp\rangle = \langle \bv,\Cmp\rangle^{2} + \sum_{t=2}^{T}\|u_{t}-u_{t-1}\|^{2},
\end{align*}
where $\bv\ne \zeros\in \R^{T}$. \footnote{Note that any such $\bM$ is
unique. Indeed, if there are positive definite matrices $\bM_{1}\in\R^{T\times T}$
and $\bM_{2}\in\R^{T\times T}$ such that
$\norm{\Cmp}^{2}_{\bM_{1}}=\inner{\bv,\Cmp}^{2}+P_{T}^{\norm{\cdot}^{2}_{2}}=\norm{\Cmp}^{2}_{\bM_{2}}$
for
all $\Cmp\in\R^{T}$, then $\inner{\Cmp,(\bM_{1}-\bM_{2})\Cmp}=0$ and hence
$\bM_{1}=\bM_{2}$  since $\bM_{1}$ and $\bM_{2}$ are positive definite.}
Note such a constant offset term is unavoidable: it is what captures
the static regret guarantee in the case where $u_{1}=\ldots=u_{T}=u$. Proposition
2 considers $\bv=(0,\ldots,0,1)$ to get
$\|\Cmp\|^{2}_{\bM}=\|u_{T}\|^{2}+\sum_{t=2}^{T}\|u_{t}-u_{t-1}\|^{2}$, though below we
will show that any vector $\bv$ would still lead to $\mathrm{Tr}(\bM^{-1})=\Omega(T^{2})$.

It is clear that the only way to construct expressions of the
form above is via matrices $\bSigma$ satisfying
\[\bSigma\Cmp=c\begin{pmatrix}u_{1}-u_{2}\\u_{2}-u_{3}\\\vdots\\u_{T-1}-u_{T}\\ \langle \bv,\Cmp\rangle\end{pmatrix},\]
where $c\in\Set{-1,1}$ and the order of the
rows indices of the vector can be permuted without loss of generality. In particular, the only matrices
that can produce these expressions (again noting that the rows can be permuted
without loss of generality) are
of the form
\begin{align*}
\bSigma = c\begin{pmatrix}1&-1&0&0&\dots&0&0\\
            0&1&-1&0&\dots&0&0\\
            0&0&1&-1&\dots&0&0\\
            \vdots& & &\ddots&&&\\
            0&0&0&0&\dots&1&-1\\
            \hline
            v_{1}&v_{2}&v_{3}&v_{4}&\ldots&v_{T-1}&v_{T}
          \end{pmatrix}
  =:
  c\begin{pmatrix}
    \bDelta\\
    \hline\\
    \bv^{\top}
  \end{pmatrix},
\end{align*}
so $\bM=\bSigma^{\top}\bSigma=\bDelta^\top\bDelta + \bv \bv^\top$.
Moreover, $\bDelta^{\top}\bDelta$ is a symmetric matrix
with a unique zero eigenvalue (corresponding to
vectors in the span of $\ones=(1,\ldots,1)\in\R^{T}$),
so
applying
\Cref{lemma:perturbation-bound},
\begin{align*}
  \Tr(\bM^{\inv})=\Tr((\bDelta^{\top}\bDelta + \bv\bv^{\top})^{\inv})
  &\ge
    \norm{\bv}^{2}+\sum_{t=2}^{T}\frac{1}{\lambda_{t}(\bDelta^{\top}\bDelta)}.
\end{align*}
Now, define $\bv_{0}=(0,\ldots,0,1)\in\R^{T}$ and observe that
$\bM_{0}:=\bDelta^{\top}\bDelta+\bv_{0}\bv_{0}^{\top}$ is precisely the
matrix studied in \Cref{prop:sqr-trade-off}. We have via the
interlacing property of rank-1 updates to symmetric matrices that
$\lambda_{t}(\bDelta^{\top}\bDelta)\le \lambda_{t}(\bM_{0})$ \citep[Theorem
8.1.8]{golub2013matrix},
so overall we have
\begin{align*}
  \Tr(\bM^{\inv})
  &\ge
    \sum_{t=2}^{T}\frac{1}{\lambda_{t}(\bM_{0})}
    =
    \sum_{t=2}^{T}\lambda_{t}(\bM_{0}^{\inv})\\
  &=
    \Tr(\bM_{0}^{\inv})-\lambda_{\max}(\bM_{0}^{\inv})\\
  &\ge
    \Tr(\bM_{0}^{\inv})-\frac{9}{10}\Tr(\bM_{0}^{\inv})\\
  &=
    \frac{1}{10}\Tr(\bM_{0}^{\inv}) = \frac{1}{10}\frac{T(T+1)}{2}
\end{align*}
where the last inequality applies \Cref{lemma:new-eigen-bound} to bound
$\lambda_{\max}(\bM_{0}^{\inv})$ and recalls $\Tr(\bM_{0}^{\inv})=\frac{T(T+1)}{2}$
from \Cref{prop:sqr-trade-off}.

Hence, the variance penalty will still be $\Omega(T^{2})$ regardless of the choice of
bias $\inner{\bv,\Cmp}^{2}$ in the variability measure. Combined with
our lower bound in \Cref{thm:pf-lb}, it follows
that adapting to the squared path-length necessarily implies a variance penalty
of $\Tr(\bM^{\inv})\ge \Omega(T^{2})$, leading to a vacuous regret upper bound.

\subsubsection{Proof of Lemma~\ref{lemma:perturbation-bound}}%
\label{app:perturbation-bound}
\PerturbationBound*
\begin{proof}
  Let $\bA:=\bB+\bv\bv^{\top}$. Since $\bB$ is symmetric, we have via
  the interlacing property that there is an $a_{1},\ldots,a_{T}\ge 0$
  such that $\sumtT a_{t}=\norm{\bv}^{2}$ and
  $\lambda_{t}(\bA)=\lambda_{t}(\bB)+a_{i}$
  \cite[see, \textit{e.g.}, Theorem 8.1.8 in ][]{golub2013matrix}. Hence,
  \begin{align*}
    \Tr((\bB+\bv\bv^{\top})^{\inv})
    &=
      \Tr(\bA^{\inv})
    =
      \sumtT \lambda_{t}(\bA^{\inv}) = \sumtT \frac{1}{\lambda_{t}(\bA)} \\
    &= \sumtT \frac{1}{\lambda_{t}(\bB)+a_{i}}
    \ge
      \min_{\substack{a_{1},\ldots,a_{T}\ge 0\\\sumtT a_{t}=\norm{\bv}^{2}}}\sumtT \frac{1}{\lambda_{t}(\bB)+a_{i}}~.
  \end{align*}

  To analyze the constrained optimization in the last line, let
  $\alpha_{1},\ldots,\alpha_{T}\ge 0$, $\beta\in\R$, and define the Lagrangian
  \begin{align*}
    L(a_{1},\ldots,a_{T},\alpha_{1},\ldots,\alpha_{T},\beta)=\sumtT \frac{1}{\lambda_{t}(\bB)+a_{t}}-\sumtT \alpha_{t}a_{t} + \beta\brac{\sumtT a_{t}-\norm{\bv}^{2}}.
  \end{align*}
  For any $t$, we have
  \begin{align*}
    \frac{\partial L}{\partial a_{t}} = \frac{-1}{(\lambda_{t}(\bB)+a_{t})^{2}}-\alpha_{t}+\beta = 0 \iff a_{t} = \frac{1}{\sqrt{\beta-\alpha_{t}}} - \lambda_{t}(\bB)~.
  \end{align*}
  Plugging this into the dual
  $D(\alpha_{1},\ldots,\alpha_{T},\beta)=\min_{a_{1},\ldots,a_{T}}L(a_{1},\ldots,a_{T},\alpha_{1},\ldots,\alpha_{T},\beta)$
  we have
  \begin{align*}
    D(\alpha_{1},\ldots,\alpha_{T},\beta)
    &=
      \sumtT \sqrt{\beta-\alpha_{t}} -\sumtT \alpha_{t}\brac{\frac{1}{\sqrt{\beta-\alpha_{t}}}-\lambda_{t}(\bB)} + \beta\brac{\sumtT \frac{1}{\sqrt{\beta-\alpha_{t}}}-\lambda_{t}(\bB)-\norm{\bv}^{2}}\\
    &=
      \sumtT \sqrt{\beta-\alpha_{t}} +\sumtT (\beta-\alpha_{t})\frac{1}{\sqrt{\beta-\alpha_{t}}}+\sumtT (\alpha_{t}-\beta)\lambda_{t}(\bB)-\beta\norm{\bv}^{2}\\
    &=
      2\sumtT \sqrt{\beta-\alpha_{t}}+\sumtT (\alpha_{t}-\beta)\lambda_{t}(\bB)-\beta\norm{\bv}^{2}~.
  \end{align*}
  The derivatives of the dual \wrt{} $\alpha_{t}$ are
  \begin{align*}
    \frac{\partial D}{\partial\alpha_{t}} = \frac{-1}{\sqrt{\beta-\alpha_{t}}}+\lambda_{t}(\bB)~.
  \end{align*}
  Observe that
  for $\lambda_{1}(\bB)=0$, we have
  $\frac{\partial D}{\partial\alpha_{t}} = - \frac{1}{\sqrt{\beta-\alpha_{t}}}\le 0$, so $D$ is
  decreasing in $\alpha_{1}$, so the dual is maximized when $\alpha_{1}=0$.
  Using the relation $a_{1}=\frac{1}{\sqrt{\beta-\alpha_{1}}}-\lambda_{1}(\bB)$ above we have
  $a_{1}=\frac{1}{\sqrt{\beta-\alpha_{1}}}-\lambda_{1}(\bB)=\frac{1}{\sqrt{\beta}}$. Equating the
  other derivatives for $t>1$ to zero we have
  \begin{align*}
    &\frac{1}{\sqrt{\beta-\alpha_{t}}}=\lambda_{t}(\bB)\implies \lambda_{t}(\bB) +a_{t}=\lambda_{t}(\bB)\\
    &\qquad\implies a_{t}=0 \quad\forall t>1
  \end{align*}
  where we used the relationship
  $a_{t}=\frac{1}{\sqrt{\beta-\alpha_{t}}}-\lambda_{t}(\bB)$ from above. Finally,
  the optimal $\beta$ is such that
  $\sumtT a_{t}=\frac{1}{\sqrt{\beta}}=\norm{\bv}^{2}$, so overall we have
  \begin{align*}
    \min_{\substack{a_{1},\ldots,a_{T}\ge 0\\ \sumtT a_{t}=\norm{\bv}^{2}}} \sum_{t=1}^{T}\frac{1}{\lambda(\bB)+a_{i}}
    &=
    \frac{1}{\sqrt{\beta-\alpha_{1}}}+\sum_{t=2}^{T}\frac{1}{\lambda_{t}(\bB)+a_{i}}
    =
    \frac{1}{\sqrt{\beta}}+\sum_{t=2}^{T}\frac{1}{\lambda_{t}(\bB)}\\
    &=
    \norm{\bv}^{2} + \sum_{t=2}^{T}\frac{1}{\lambda_{t}(\bB)}~. \qedhere
  \end{align*}

\end{proof}
\subsubsection{Proof of Lemma~\ref{lemma:new-eigen-bound}}%
\label{app:new-eigen-bound}
\NewEigenBound*
\begin{proof}
  The matrix $\bM^{\inv}=\bSigma^{\inv}\bSigma^{-\top}$ is symmetric and
  positive definite, hence has real eigenvalues.
  The eigenvalues of $\bM^{\inv}$ can be bound
  in terms of its trace as follows (see, \eg{}, Theorem 2.1
  \citet{wolkowicz1980bounds}, provided for convenience in \Cref{thm:eigen-bound}):
  \begin{align*}
    \lambda_{\max{}}(\bM^{\inv})\le \frac{\Tr\brac{\bM^{\inv}}}{T}+\sqrt{(T-1)\sbrac{\frac{\Tr\brac{\bM^{-\top}\bM^{\inv}}}{T}-\brac{\frac{\Tr\brac{\bM^{\inv}}}{T}}^{2}}}~.
  \end{align*}

    Next, observe that by
    \Cref{lemma:diff-op-properties},
    matrix $\bSigma^{\inv}$ is an upper-triangular matrix of all $1$'s,
    so that
    \begin{align*}
    [\bM^{\inv}]_{ij}
    &=\sbrac{\bSigma^{-1}\bSigma^{-\top}}_{ij} = \sum_{k\in[t]} \Sigma_{ik}^{\inv}\Sigma_{kj}^{-\top}
    =
        \sum_{k\in[T]}\Sigma_{ik}^{^{\inv}}\Sigma_{jk}^{\inv}\\
    &=
        \sum_{k\in [T]}\indicator\Set{k\ge i}\indicator\Set{k\ge j} = T-\max\Set{i,j}+1,
    \end{align*}
    Hence,
    \begin{align*}
    \Tr\brac{\bM^{\inv}}= \sumtT [\bM^{\inv}]_{ii}=\sumtT (T-t+1)=\sumtT t=\frac{T(T+1)}{2}.
    \end{align*}
    Moreover,
    \begin{align*}
    \Tr\brac{\bM^{-\top}\bM^{\inv}}
    &=
        \sumtT[\bM^{-\top}\bM^{\inv}]_{tt}=\sumtT \sum_{k=1}^{t}M^{-\top}_{tk}M^{\inv}_{kt}\\
    &=\sumtT\sum_{k=1}^{t}(M^{\inv}_{kt})^{2} = \sumtT \sum_{k=1}^{t}(T-\Max{t,k}+1)^{2}\\
    &=
        \frac{T(T+1)^{2}(T+2)}{12}~.
    \end{align*}
    Thus,
    \begin{align*}
    \sbrac{\frac{\Tr\brac{\bM^{-\top}\bM^{\inv}}}{T}-\brac{\frac{\Tr\brac{\bM^{\inv}}}{T}}^{2}}
    &=
        \frac{(T+1)^{2}(T+2)}{12}-\frac{(T+1)^{2}}{4}\\
        &=
        \frac{(T+1)^{2}}{4}\sbrac{\frac{T+2}{3}-1}\\
        &=
        \frac{(T+1)^{2}}{4}\frac{T-1}{3}~.
    \end{align*}
    Overall, $\lambda_{\max}(\bM^{\inv})$ is bounded by
    \begin{align*}
    \lambda_{\max}(\bM^{\inv})
    &\le
        \frac{\Tr\brac{\bM^{\inv}}}{T}+\sqrt{(T-1)\frac{(T+1)^{2}}{4}\frac{T-1}{3}}\\
    &=
        \frac{T+1}{2}+\frac{(T+1)(T-1)}{2\sqrt{3}}\\
    &=
        \frac{T(T+1)}{2\sqrt{3}}+\frac{T+1}{2}\sbrac{1-\frac{1}{\sqrt{3}}}\\
    &\le
        \frac{T(T+1)}{2\sqrt{3}}+\frac{T(T+1)}{2}\frac{1}{4}
        \le
        \frac{9}{10}\frac{T(T+1)}{2} = \frac{9}{10}\Tr(\bM^{\inv}),
    \end{align*}
    where the last line observes that $1-\frac{1}{\sqrt{3}}\le \half\le \frac{T}{4}$ for
    $T\ge 2$ and the fact that $\frac{1}{\sqrt{3}}+\frac{1}{4}\approx 0.83\le \frac{9}{10}$.
\end{proof}

%auto-ignore
\section{Proofs for Section~\ref{sec:applications} (\SecApplications)}

\subsection{Details on the 1-Dimensional Reduction}
\label{app:1d-redux}

In this section, for completeness we provide the details of 
the 1-dimensional reduction
of \citet{cutkosky2018black}, specialized to
dual weighted-norm pairs $(\norm{\cdot}_{\bM},\norm{\cdot}_{\bM^{\inv}})$ as well as its regret guarantee.

For concreteness, we choose adaptive FTRL with AdaGrad-norm stepsizes~\citep{StreeterM10}
as the direction learner. For simplicity we use the scale-free version of
\cite{orabona2018scale}, so that the direction learner's update
is slightly simpler, not requiring prior knowledge of the Lipschitz constant $\mathfrak{G}\ge \norm{\Gt}_{\bM^{\inv}}$.

Using \citet[Theorem 2]{cutkosky2018black}, we have that the regret of \Cref{alg:1d-redux} is equal to
\[
R_T(\Cmp)=R^{\cA}_T(\|\Cmp\|_{\bM})+\|\cmp\|_{\bM} R^\text{direction}_T\left(\frac{\Cmp}{\|\Cmp\|_{\bM}}\right), \ \forall \Cmp \in \R^{d T},
\]
where $R^{\cA}_T$ is the regret of $\cA$ over a sequence of $G$-Lipschitz linear losses and $R_T^\text{direction}$ is the regret of (scale-free) adaptive FTRL with a feasible set equal to the unitary ball defined by $\|\cdot\|_{\bM}$.

Choosing the algorithm $\cA$ to be \citep[Algorithm 1]{jacobsen2022parameter}, we have
\[
R^{\cA}_{T}(\|\Cmp\|_{\bM})
\le \scO\brac{\mathfrak{G} \epsilon + \norm{\Cmp}_{\bM}\sbrac{\sqrt{V_{T}\Log{\frac{\norm{\Cmp}_{\bM}\sqrt{V_{T}}\Lambda_{T}}{\mathfrak{G}\epsilon}+1}}\maxOp\mathfrak{G}\Log{\frac{\norm{\Cmp}_{\bM}\sqrt{V_{T}}\Lambda_{T}}{\epsilon\mathfrak{G}}}}},
\]
where $V_{T}=\sumtT \norm{\Gt}^{2}_{\bM^{\inv}}$ and $\Lambda_{T}=\log^{2}(\sumtT \norm{\Gt}^{2}_{\bM^{\inv}}/\mathfrak{G}^{2})\le \scO(\log^{2}T)$.

Focusing now on the regret of the direction learner,
define the distance generating function $\psi(\tilde{\bx})=\frac{1}{2} \|\tilde{\bx}^2\|_{\bM}$. Using \citep[Theorem 4.3]{Orabona19}, we have that $\psi$ is 1-strongly convex \wrt{} $\|\cdot\|_{\bM}$.
Hence, using the regret guarantee of Scale-free FTRL, \ie{}, Theorem 1 of
\citet{orabona2018scale}, for any $\tilde\bv\in\R^{dT}$ such that $\norm{\tilde\bv}_{\bM}\le 1$ the regret of the direction learner is
\begin{align*}
R_{T}^{\text{Direction}}(\tilde\bv)
  &\le
  \sbrac{\half\norm{\tilde\bv}_{\bM}^{2}+2.75}\sqrt{\sumtT \norm{\Gt}^{2}_{\bM^{\inv}}} + 3.5 \max_{t\le T}\norm{\Gt}_{\bM^{\inv}}
  \le
  \scO\brac{\sqrt{\sumtT \norm{\Gt}^{2}_{\bM^{\inv}}}}.
\end{align*}
Applying this with $\tilde\bv=\frac{\Cmp}{\norm{\Cmp}_{\bM}}$
and combining with the previous two displays leads to the
bound stated in the proof of \Cref{thm:simple-dynamic}.

    %auto-ignore

\subsection{The Haar Matrices and their Properties}%
\label{app:haar}

In this section we provide some useful supporting lemmas related
to the Haar matrices $\bH_{n}$. We first introduce the
\emph{Haar basis vectors},  which make up the columns of the matrix $\bH_{n}$.
Throughout this section we assume for simplicity that $T$ is a power of 2.
\begin{definition}\label{def:haar-basis}
  For any $\tau\in\Set{2^{i}:i=1:\log_{2}(T)}$ and
  $i\in[T/\tau]$, the Haar basis vector
  at timescale $\tau$ and location $i$
  is the vector in $\R^{T}$ with entries
  \begin{align}
    [\bh_{i}^{(\tau)}]_{t} = \begin{cases}
                        1&\text{if }t\in[\half \tau(i-1)+1,\half \tau i]\\
                        -1&\text{if }t\in[\half \tau i+1, \tau i]\\
                        0&\text{otherwise}
                  \end{cases}\label{eq:haar-basis-vectors}
  \end{align}
\end{definition}
The Haar basis vectors are often arranged into the columns of a matrix as follows:
\begin{align*}
  \cH_{n}=\pmat{\bh_{0}&\bh^{(T)}_{1}&\bh^{(T/2)}_{1}&\bh^{(T/2)}_{2}&\bh^{(T/4)}_{1}&\bh_{2}^{(T/4)}&\bh^{(T/4)}_{3}&\bh^{(T/4)}_{4}&\cdots&\bh^{(2)}_{T/2}},
\end{align*}
where $\bh_{0}=(1,1,\ldots,1)^{\top}\in\R^{T}$. This matrix is referred to as the
(unnormalized) Haar basis matrix of order $n=\log_{2}(T)$. It is well-known that
$\bH_{n}$ has the following equivalent recursive form\ \citep{steeb1997matrix,falkowski1998generalized,stankovic2003haar}:
\begin{align}
  \cH_{0} &= (1),\nonumber\\
  \cH_{n} &= \begin{pmatrix}
            \cH_{n-1}\otimes\begin{pmatrix}1\\1\end{pmatrix}& \bI_{2^{n-1}}\otimes\begin{pmatrix}1 \\-1\end{pmatrix}
          \end{pmatrix}~.\label{eq:haar}
\end{align}
So, for instance, we have
\begin{align*}
  \cH_{1}
  &= \begin{pmatrix}
           \cH_{0}\otimes \begin{pmatrix}1\\1\end{pmatrix}& \bI_{2^{0}}\otimes\begin{pmatrix}1 \\-1\end{pmatrix}
        \end{pmatrix}
  = \begin{pmatrix}
           (1)\otimes \begin{pmatrix}1\\1\end{pmatrix}& (1)\otimes\begin{pmatrix}1 \\-1\end{pmatrix}
        \end{pmatrix}
  = \begin{pmatrix}
      1 &1\\
      1&-1
    \end{pmatrix},\\
  \cH_{2}
  &=
    \begin{pmatrix}
      \begin{pmatrix}
        1 &1\\
        1&-1
      \end{pmatrix}
      \otimes\begin{pmatrix}1\\1\end{pmatrix}
      & \begin{pmatrix}1&0\\0&1\end{pmatrix}\otimes\begin{pmatrix}1 \\-1\end{pmatrix}
    \end{pmatrix}
    =\begin{pmatrix}
       1 &1& 1&0\\
       1& 1& -1&0\\
       1& -1 &0&1\\
       1& -1&0&-1
    \end{pmatrix},
\end{align*}
and so on. For our purposes, we will primaly work in terms
of the matrices $\bH_{n}$ rather than the basis vectors $\bh_{i}^{(\tau)}$.
The
main utility of defining the basis vectors $\bh_{i}^{(\tau)}$
is that their definition easily implies the
following useful result, which states that
the Haar basis vectors are \emph{sparsely supported}
\wrt{} time.
\begin{restatable}{proposition}{SparseSupport}\label{prop:sparse-support}
  Let $n=\log_{2}T$ and let $\bH_{n}\in\R^{T\times T}$ be the unnormalized Haar
  basis matrix of order $n$. Then for any $t\in[T]$, there
  at most $1+\log T$ indices $i$ for which $[\bH_{n}]_{t,i}\ne 0$.
\end{restatable}
The proof follows immediately from \Cref{def:haar-basis} (\ie{}, any $t$ can
fall into only one of the intervals covered at each of the $\log_{2}(T)$
time-scales) and accounting for the additional column $\bh_{0}$ of
all $1$'s.

In what follows, we will also use the following
well-known relationship between the vec operator and
the Kronecker product
(see, \textit{e.g.},
\citet[Chapter 2.11]{steeb1997matrix}).
\begin{restatable}{proposition}{KroneckerVec}\label{prop:kronecker-vec}
  Let $\bA$, $\bB$, and $\bC$ be matrices of appropriate dimensions
  such that the product $\bA \bB \bC$ exists. Then, $\vecOp(\bA \bB \bC)=(\bC^{\top}\otimes \bA)\vecOp(\bB)$.
\end{restatable}

The following three lemmas will be used to prove the guarantees of the algorithm
characterized in \Cref{sec:haar} (\Cref{prop:haar-trade-off,prop:haar-cmput}).
\begin{restatable}{lemma}{InverseHaar}\label{lemma:inverse-haar}
  Let $n=\log_{2}(T)$, $\bv=(v_{1},\ldots,v_{T})^{\top}\in\R^{T}$, and let $\cH_{n}$ be the unnormalized
  Haar basis matrix of order $n$. Then
  \begin{align*}
    \cH_{n}^{T}\bv= \pmat{\cH_{n-1}^{\top}\bv_{+}\\I_{2^{n-1}}\bv_{-}},
  \end{align*}
  where
  \begin{align*}
    \bv_{+} =\pmat{v_{1}+v_{2}\\v_{3}+v_{4}\\\vdots\\v_{T-1}+v_{T}},\quad
    \bv_{-}=\pmat{v_{1}-v_{2}\\v_{3}-v_{4}\\\vdots\\ v_{T-1}-v_{T}}.
  \end{align*}
\end{restatable}
\begin{proof}
  From \Cref{eq:haar}, we have that
  \begin{align*}
    \cH^{\top}_{n}\bv
    &=
      \pmat{\cH_{n-1}\otimes \pmat{1\\1}& I_{2^{n-1}}\otimes\pmat{1\\-1}}^{\top}\bv
    =
      \pmat{\cH_{n-1}^{\top}\otimes \pmat{1&1}\\ I_{2^{n-1}}\otimes\pmat{1&-1}}\bv\\
    &=
      \pmat{\sbrac{\cH_{n-1}^{\top}\otimes \pmat{1&1}}\bv\\ \sbrac{I_{2^{n-1}}\otimes\pmat{1&-1}}\bv}~.
  \end{align*}
  Moreover, leveraging \Cref{prop:kronecker-vec} we have
  \begin{align*}
    \cH^{\top}_{n}\bv
    &=
      \pmat{
      \vecOp\brac{\pmat{1&1}\pmat{v_{1}&v_{3}&\cdots&v_{T-1}\\ v_{2}&v_{4}&\cdots&v_{T}}\cH_{n-1}}\\
      \vecOp\brac{\pmat{1&-1}\pmat{v_{1}&v_{3}&\cdots&v_{T-1}\\ v_{2}&v_{4}&\cdots&v_{T}}\cH_{n-1}}\\
    }\\
    &=
      \pmat{
      \vecOp\brac{\overbrace{\pmat{v_{1}+v_{2}&v_{3}+v_{4}&\cdots&v_{T-1}+v_{T}}}^{=\bv_{+}^{\top}}\cH_{n-1}}\\
      \vecOp\brac{\underbrace{\pmat{v_{1}-v_{2}&v_{3}-v_{4}&\cdots&v_{T-1}-v_{T}}}_{=\bv_{-}^{\top}}I_{2^{n-1}}}\\
    }\\
    &=
      \pmat{
      \cH_{n-1}^{\top}\bv_{+}\\
      I_{2^{n-1}}\bv_{-}
    }~. \qedhere
  \end{align*}
\end{proof}

\begin{restatable}{lemma}{HHt}\label{lemma:hht}
  Let $\cH_{n}$ be the unnormalized Haar basis matrix of order $n$.
  Then, $\cH_{n}\cH_{n}^{\top}$ satisfies
  \begin{align}
    \cH_{n}\cH_{n}^{\top}
    &=
    \cH_{n-1}\cH_{n-1}^{\top}\otimes \begin{pmatrix} 1&1\\1&1\end{pmatrix} + I_{2^{n-1}}\otimes \begin{pmatrix}1 &-1\\-1&1\end{pmatrix}\label{eq:hht:kronecker}\\
    &=
      \begin{pmatrix}
        \cH_{n-1}\cH_{n-1}^{\top}+\ones_{2^{n-1}}&\zeros_{2^{n-1}}\\
        \zeros_{2^{n-1}}&\cH_{n-1}\cH_{n-1}^{\top}+\ones_{2^{n-1}},
      \end{pmatrix},\label{eq:hht:recursive}
  \end{align}
  where $\ones_{2^{n-1}}$ and $\zeros_{2^{n-1}}$ are $2^{n-1}\times 2^{n-1}$ matrices of 1's and 0's respectively.
\end{restatable}
\begin{proof}
  For brevity, let us denote $\bB_{n}=\cH_{n}\cH_{n}^{\top}$.
  The first equality follows from elementary properties of
  block matrices and the Kronecker product: using the
  recursive form of $\cH_{n}$, we have
  \begin{align*}
    \bB_{n}
    &=
      \cH_{n}\cH_{n}^{\top}\nonumber\\
    &=
      \begin{pmatrix}\cH_{n-1}\otimes\begin{pmatrix}1\\1\end{pmatrix}& \bI_{2^{n-1}}\otimes\begin{pmatrix}1\\-1\end{pmatrix}\end{pmatrix}
      \begin{pmatrix}\cH_{n-1}^{\top}\otimes\begin{pmatrix}1&1\end{pmatrix}\\ \bI_{2^{n-1}}\otimes\begin{pmatrix}1&-1\end{pmatrix}\end{pmatrix}\nonumber\\
    &=
      \cH_{n-1}\otimes\begin{pmatrix}1\\1\end{pmatrix}\cH_{n-1}^{\top}\otimes\begin{pmatrix}1&1\end{pmatrix}+\bI_{2^{n-1}}\otimes\begin{pmatrix}1\\-1\end{pmatrix}\bI_{2^{n-1}}\otimes\begin{pmatrix}1&-1\end{pmatrix}\nonumber\\
    &=
      \cH_{n-1}\cH_{n-1}^{\top}\otimes\begin{pmatrix}1\\1\end{pmatrix}\begin{pmatrix}1&1\end{pmatrix}+\bI_{2^{n-1}}\otimes\begin{pmatrix}1\\-1\end{pmatrix}\begin{pmatrix}1&-1\end{pmatrix}\nonumber\\
    &=
      \cH_{n-1}\cH_{n-1}^{\top}\otimes\begin{pmatrix}1&1\\1&1\end{pmatrix}+\bI_{2^{n-1}}\otimes\begin{pmatrix}1&-1\\-1&1\end{pmatrix}\nonumber\\
    &=
      \bB_{n-1}\otimes\begin{pmatrix}1&1\\1&1\end{pmatrix}+\bI_{2^{n-1}}\otimes\begin{pmatrix}1&-1\\-1&1\end{pmatrix}~.
  \end{align*}
  To get the second expression, let us proceed by induction.
  We have $\bB_{0}=(1)$ and
  \begin{align*}
    \bB_{1}=\cH_{1}\cH_{1}^{\top}
    =
      \begin{pmatrix}
        1&1\\
        1&-1
      \end{pmatrix}
      \begin{pmatrix}
        1&1\\
        1&-1
      \end{pmatrix}^{\top}
    =
      \begin{pmatrix}
        2&0\\
        0&2
      \end{pmatrix}
    =
      \begin{pmatrix}
        \bB_{0}+\ones_{1}&\zeros_{1}\\
        \zeros_{1}&\bB_{0}+\ones_{1}
      \end{pmatrix}~.
  \end{align*}
  Next, let us assume that $\bB_{n}$
  satisfies
  \begin{align*}
    \bB_{n}
    &=
    \begin{pmatrix}
      \bB_{n-1}+\ones_{2^{n-1}}&\zeros_{2^{n-1}}\\
      \zeros_{2^{n-1}}&\bB_{n-1}+\ones_{2^{n-1}}\\
    \end{pmatrix}~.
  \end{align*}
  Then, applying the recursive form \Cref{eq:hht:kronecker} for
  $\bB_{n+1}$, we have
  \newcommand{\IplusP}{\begin{pmatrix}1&1\\1&1\end{pmatrix}}
  \newcommand{\IminusP}{\begin{pmatrix}1&-1\\-1&1\end{pmatrix}}
  \begin{align*}
    \bB_{n+1}
    &=
      \bB_{n}\otimes\IplusP + \bI_{2^{n}}\otimes\IminusP\\
    &=
      \begin{pmatrix}\bB_{n-1}+\ones_{2^{n-1}}&\zeros_{2^{n-1}}\\\zeros_{2^{n-1}}&B_{n-1}+\ones_{2^{n-1}}\end{pmatrix}\otimes\IplusP+ \bI_{2^{n}}\otimes\IminusP\\
    &=
      \begin{pmatrix}
        \bB_{n-1}\otimes\IplusP + \ones_{2^{n-1}}\otimes \IplusP&\zeros_{2^{n}}\\
        \zeros_{2^{n}}&\bB_{n-1}\otimes\IplusP + \ones_{2^{n-1}}\otimes \IplusP
      \end{pmatrix}\\
      &\qquad+
      \begin{pmatrix}
        \bI_{2^{n-1}}\otimes \IminusP&\zeros_{2^{n}}\\
        \zeros_{2^{n}}&\bI_{2^{n}}\otimes\IminusP
      \end{pmatrix}\\
    &=
      \begin{pmatrix}
        \bB_{n}+ \ones_{2^{n}}&\zeros_{2^{n}}\\
        \zeros_{2^{n}}&\bB_{n}+ \ones_{2^{n}},
      \end{pmatrix},
  \end{align*}
  where the last line observes that $\ones_{2^{n-1}}\otimes\IplusP=\ones_{2^{n}}$
  and that after adding the two
  block matrices, the top left and bottom right blocks
  are both
  \begin{align*}
    \bB_{n-1}\otimes \IplusP + \bI_{2^{n-1}}\otimes \IminusP + \ones_{2^{n}}
    =
    \bB_{n}+\ones_{2^{n}},
  \end{align*}
  via \Cref{eq:hht:kronecker}. Hence, the stated result follows
  by induction.
\end{proof}

Now using this, we have the following bound on the
norm of the high-dimensional surrogate losses.
\begin{restatable}{lemma}{HaarGradientBound}\label{lemma:haar-gradient-bound}
  Let $n=\log_{2}(T)$, $\e_{t}$ be the $t^{\text{th}}$ standard basis vector of $\R^{T}$,
  and for $\gt\in\R^{d}$ let $\G_{t}=\et\otimes \gt\in\R^{dT}$. Let $\cH_{n}$
  be a Haar matrix of order $n$ and let $\bB=\cH_{n}\otimes I_{d}$
  be it's block extension to sequence in $\R^{d}$. Then, we have
  \begin{align*}
    \norm{\Gt}_{\bB\bB^{\top}}^{2}= (\log T+1)\norm{\gt}^{2}_{2}~.
  \end{align*}
\end{restatable}
\begin{proof}
  Using \Cref{lemma:gradient-bound}, we have that
  \begin{align*}
    \norm{\Gt}^{2}_{\bB \bB^{\top}} = \sbrac{\bH_n \bH_n^{\top}}_{tt}\norm{\gt}^{2}~.
  \end{align*}
  Moreover, using \Cref{eq:hht:recursive} it can easily be seen that
  the diagonal entries of $\bH_n \bH_n^{\top}$ are  $\log_{2}T+1$,
  so we have
  \[
    \norm{\Gt}^{2}_{\bB \bB^{\top}}\le (1+\log_{2} T)\norm{\gt}^{2}_{2}~. \qedhere
  \]
\end{proof}

\subsection{Proof of Proposition~\ref{prop:haar-trade-off}}%
\label{app:haar-trade-off}
\HaarTradeOff*
\begin{proof}
  The proof of the claim $\norm{\Gt}^{2}_{\bM^{\inv}}=\norm{\Gt}^{2}_{\cH\cH^{\top}}=\norm{\gt}^{2}_{2}\sbrac{\log_{2}(T)+1}$ is provided in
  \Cref{lemma:haar-gradient-bound}.

  To see the form of $\norm{\Cmp}_{\bM}^{2}$, let us first write
  \begin{align*}
    \norm{\Cmp}^{2}_{M}=\inner{\Cmp, [ \bH\bH^{\top}]^{\inv}\Cmp}=\inner{\bH^{\inv}\Cmp,\bH^{\inv}\Cmp}=\norm{\bH^{\inv}\Cmp}_{2}^{2}~.
  \end{align*}
  The result then follows by showing that
  \begin{align}
    \bH^{\inv}\Cmp
    &=\half
      \begin{pmatrix}
        2\cmpbar\\
        \cmpbar_{1}^{(T/2)}-\cmpbar_{2}^{(T/2)}\\
        \cmpbar_{1}^{(T/4)}-\cmpbar_{2}^{(T/4)}\\
        \cmpbar_{3}^{(T/4)}-\cmpbar_{4}^{(T/4)}\\
        \vdots\\
        \cmp_{1}-\cmp_{2}\\
        \cmp_{3}-\cmp_{4}\\
        \vdots\\
        \cmp_{T-1}-\cmp_{T}
      \end{pmatrix},\label{eq:haar-trade-off:HU}
  \end{align}
  so that
  \begin{align*}
    \norm{\bH^{\inv}\Cmp}_{2}^{2}
    &=
      \underbrace{
        \norm{\cmpbar}^{2}_{2}
      }_{\localP(T)}
      + \frac{1}{4}\underbrace{
        \norm{\cmpbar_{1}^{(T/2)}-\cmpbar_{2}^{(T/2)}}^{2}_{2}
      }_{\localP(T/2)}
    +\underbrace{
      \frac{1}{4}\norm{\cmpbar_{1}^{(T/4)}-\cmpbar_{2}^{(T/4)}}^{2}_{2}
      +\frac{1}{4}\norm{\cmpbar_{3}^{(T/4)}-\cmpbar_{4}^{(T/4)}}^{2}_{2}
    }_{\localP(T/4)}\\
    &\qquad
      +\ldots+
      \underbrace{\frac{1}{4}\norm{\cmp_{1}-\cmp_{2}}^{2}_{2}+\frac{1}{4}\norm{\cmp_{3}-\cmp_{4}}^{2}_{2}
      +\ldots+\frac{1}{4}\norm{\cmp_{T-1}-\cmp_{T}}^{2}_{2}}_{=\localP(1)},
  \end{align*}
  where for brevity we have dropped the argument $\vec{\cmp}$ on
  $\localP(\vec{\cmp},\tau)$.

  \Cref{eq:haar-trade-off:HU} is best shown
  via example; the general case is
  mostly a tedius exercise which we provide at the end.
  Assume $T=4$, then the Haar matrix of order
  $n=\log_{2}(T)=2$ is
  \begin{align*}
    \cH_{2}
    &=\begin{pmatrix}
       1 &1& 1&0\\
       1& 1& -1&0\\
       1& -1 &0&1\\
       1& -1&0&-1
      \end{pmatrix}
    =\underbrace{
      \begin{pmatrix}
       \half &\half& \frac{1}{\sqrt{2}}&0\\
       \half & \half& \frac{-1}{\sqrt{2}}&0\\
       \half & -\half &0&\frac{1}{\sqrt{2}}\\
       \half & -\half&0&\frac{-1}{\sqrt{2}}
      \end{pmatrix}}_{=:\tilde \cH_{2}}
      \underbrace{\begin{pmatrix}
        2& 0&0&0\\
        0& 2&0&0\\
        0&0&\sqrt{2}&0\\
        0&0&0&\sqrt{2}
      \end{pmatrix}}_{=:\bD_{2}}.
  \end{align*}
  It is well-known that for any $T$ the columns of $\cH_{\log_{2}(T)}$ form an orthogonal basis
of $\R^{T}$~\cite[Chapter 6.1.1]{walnut2013introduction},
  which implies that $\tilde\cH_{2}$ is orthonormal. So,
  $\tilde \cH_{2}^{\inv}=\tilde\cH_{2}^{\top}$ and
  \begin{align*}
    \cH_{2}^{\inv}
    &=
      (\tilde\cH_{2}\bD_{2})^{\inv}=\bD_{2}^{\inv}\tilde \cH_{2}^{\inv}=\bD_{2}^{\inv}\tilde\cH_{2}^{\top}\\
    &=
      \begin{pmatrix}
        \half& 0&0&0\\
        0& \half&0&0\\
        0&0&\frac{1}{\sqrt{2}}&0\\
        0&0&0&\frac{1}{\sqrt{2}}
      \end{pmatrix}
      \begin{pmatrix}
       \half &\half& \half&\half\\
       \half & \half& -\half&-\half\\
       \frac{1}{\sqrt{2}} & -\frac{1}{\sqrt{2}} &0&0\\
        0&0& \frac{1}{\sqrt{2}}&\frac{-1}{\sqrt{2}}
      \end{pmatrix}
    =
      \begin{pmatrix}
       \frac{1}{4} &\frac{1}{4}& \frac{1}{4}&\frac{1}{4}\\
       \frac{1}{4} & \frac{1}{4}& -\frac{1}{4}&-\frac{1}{4}\\
       \frac{1}{2} & -\frac{1}{2} &0&0\\
        0&0& \frac{1}{2}&\frac{-1}{2}
      \end{pmatrix},
  \end{align*}
  which leads to \Cref{eq:haar-trade-off:HU} after applying
  the Kronecker product:
  \begin{align*}
    H^{\inv}\Cmp
    &=
      \begin{pmatrix}
       \frac{\bI_{d}}{4} &\frac{\bI_{d}}{4}& \frac{\bI_{d}}{4}&\frac{\bI_{d}}{4}\\
       \frac{\bI_{d}}{4} & \frac{\bI_{d}}{4}& -\frac{\bI_{d}}{4}&-\frac{\bI_{d}}{4}\\
       \frac{\bI_{d}}{2} & -\frac{\bI_{d}}{2} &\zeros&\zeros\\
        \zeros&\zeros& \frac{\bI_{d}}{2}&\frac{-\bI_{d}}{2}
      \end{pmatrix}
    \begin{pmatrix}
      \cmp_{1}\\
      \vdots\\
      \cmp_{T}
    \end{pmatrix}
    =
    \begin{pmatrix}
      \frac{\cmp_{1}+\cmp_{2}+\cmp_{3}+\cmp_{4}}{4}\\
      \frac{\cmp_{1}+\cmp_{2}-\cmp_{3}-\cmp_{4}}{4}\\
      \frac{\cmp_{1}-\cmp_{2}}{2}\\
      \frac{\cmp_{3}-\cmp_{4}}{2}
    \end{pmatrix}
    =
    \half
    \begin{pmatrix}
      2\bar\cmp\\
      \cmpbar^{(T/2)}_{1}-\cmpbar^{(T/2)}_{2}\\
      \cmp_{1}-\cmp_{2}\\
      \cmp_{3}-\cmp_{4}
    \end{pmatrix}.
  \end{align*}

  More generally, start with $d=1$
  begin again by factoring
  \begin{align*}
    \cH_{n}^{\inv}=\bD_{n}^{\inv}\tilde\cH_{n}^{\top} = \bD_{n}^{-2}\cH_{n}^{\top},
  \end{align*}
  where now
  $\tilde\cH_{n}$ is the normalized Haar basis matrix of order $n=\log_{2}(T)$ and
  \begin{align*}
  \bD_{n}=\Diag{\sqrt{T}, \underbrace{\sqrt{T}}_{2^{0}}, \underbrace{\sqrt{T/2},\sqrt{T/2}}_{2^{1}}, \underbrace{\sqrt{T/4},\ldots,\sqrt{T/4}}_{2^{2}},\ldots,\underbrace{\sqrt{2}, \ldots,\sqrt{2}}_{2^{n-1}}}.
  \end{align*}
  The result is then attained by
  unrolling the recursion for $\cH_{n}^{\top}\Cmp$ given by
  \Cref{lemma:inverse-haar} and factoring in
  the normalization factors
  $\bD_{n}^{-2}$. The result for $d>1$
  is then immediately implied by observing that
  the block matrix $\cH_{n}^{\inv}\otimes \bI_{d}$ will act
  upon the vector components of $\Cmp\in\R^{dT}$ in an identical way to how
  $\cH_{n}^{\inv}$ acts upon a vector of scalars.
\end{proof}

\subsection{Proof of Proposition~\ref{prop:haar-cmput}}%
\label{app:haar-cmput}
\HaarCmput*
\begin{proof}
  Note that the losses passed to
  the 1-dimensional parameter-free algorithm
  are $\inner{\tilde\bv_{t}, \Gt}=\inner{\tilde\bv_{t},\et\otimes\gt}$,
  and since $\et\otimes \gt$ has only $d$ active indices
  we can compute the 1-dimensional learner's losses
  in $\scO(d)$. As such, the
  1-dimensional learner can be implemented in
  $\scO(d)$ per-round computation.

  For the direction learner, we are to show that
  each of the relevant variables can be maintained
  using only $\scO(d \log T)$ per-round computation.

  Using \Cref{prop:haar-trade-off}, we immediately
  have $V_{\tpp}=V_{t}+\norm{\Gt}_{\bM^{\inv}}^{2}=V_{t}+(\log T+1)\norm{\gt}^{2}$,
  so $V_{\tpp}$ can be maintained using only $\scO(d)$ per-round computation (\ie,
  to compute $\|\gt\|^{2}$).

  For the scaling factor $\norm{\tilde\btheta_{\tpp}}_{\bM^{\inv}}$, observe that
  \begin{align*}
    \norm{\tilde\btheta_{\tpp}}_{\bM^{\inv}}^{2}
    &=
      \norm{\Gt}^{2}_{\bM^{\inv}} + \norm{\tilde\btheta_{t}}^{2}_{\bM^{\inv}}
      +2\inner{\tilde\btheta_{t},\bM^{\inv}\Gt}.
  \end{align*}
  Hence, we again have $\scO(d)$ per-round computation
  to compute $\norm{\Gt}^{2}_{\bM^{\inv}}$, and letting $\bh_{t}=\bH_{n}^{\top}\be_{t}$
  we can decompose the last term as
  \begin{align*}
    \inner{\tilde \btheta_{t}, \brac{\bH_{n}\bH_{n}^{\top}\otimes \bI_{d}}(\et\otimes \gt)}
    &=
      \inner{\tilde \btheta_{t}, \brac{\bH_{n}\bH_{n}^{\top}\et\otimes\gt}}\\
    &=
      \inner{\sum_{i=1}^{\tmm}\be_{i}\otimes \g_{i}, \bH_{n}\bh_{t}\otimes\gt}\\
    &=
      \sum_{i=1}^{\tmm}(\be_{i}^{\top}\otimes \g_{i}^{\top})\brac{\bH_{n}\bh_{t}\otimes \gt}\\
    &=
      \sum_{i=1}^{\tmm}\be_{i}^{\top}\bH_{n}\bh_{t}\otimes \inner{\g_{i},\gt}\\
    &=
      \sum_{i=1}^{\tmm}\inner{\bh_{i},\bh_{t}}\inner{\g_{i},\gt}\\
    &=
      \inner{\sum_{i=1}^{\tmm}\bh_{i}\inner{\g_{i},\gt},\bh_{t}}\\
    &=
      \inner{\sum_{i=1}^{\tmm}\bh_{i}\g_{i}^{\top}\gt,\bh_{t}}\\
    &=
      \inner{\gt,\underbrace{\sbrac{\sum_{i=1}^{\tmm}\g_{i}\bh_{i}^{\top}}}_{=:\bLambda_{t}}\bh_{t}}.
  \end{align*}
  From \Cref{prop:sparse-support}, for any $t$ the vector $\bh_{t}=\bH_{n}^{\top}\et$
  has only $\log T+1$ active non-zero elements by construction of
  the Haar basis, so given
  $\bLambda_{t}$, the product $\bLambda_{t}\bh_{t}$
  takes a linear combination of $\log T +1$ vectors in $\R^{d}$,
  for $\scO(d\log T )$ operations.
  Note that the variable $\bLambda_{t}$ can also be maintained
  with $\scO(d\log T )$ operations since each term is $\gt \bh_{t}^{\top}$,
  which involves updating $\scO(\log T)$ columns of $\bLambda_{\tmm}\in\R^{d\times T}$.
  Hence overall we can maintain $\norm{\tilde\btheta_{\tpp}}_{\bM^{\inv}}$
  using $\scO(d\log T )$ per-round computation.

  Lastly, consider the variable $\tilde \btheta_{\tpp}$.
  Observe that we can maintain a
  variable
  $\hat\btheta_{\tpp}=-\brac{\bH_{n}^{\top}\otimes \bI_{d}}\sum_{s=1}^{t}\G_{s}$
  using $\scO(d\log T)$ computation:
  \begin{align*}
    \hat\btheta_{\tpp}
    =-\brac{\bH_{n}^{\top}\otimes \bI_{d}}\sum_{s=1}^{t}\G_{s}
    =\hat\btheta_{t}-\brac{\bH_{n}^{\top}\be_{t}\otimes \gt}
    =\hat\btheta_{t}-\brac{\bh_{t}\otimes \gt},
  \end{align*}
  since $\bh_{t}\otimes\gt$ is a block vector containing
  $\Log{T}+1$ non-zeros blocks of length $d$.
  Hence,
  \begin{align*}
    \tilde\btheta_{\tpp}
    &=
      (\bH_{n}\bH_{n}^{\top}\otimes\bI_{n})\sum_{s=1}^{t}\g_{s}
      =
      (\bH_{n}\otimes \bI_{n})(\bH_{n}^{\top}\otimes \bI_{n})\sum_{s=1}^{t}\G_{s}\\
    &=
      (\bH_{n}\otimes \bI_{n})\hat\btheta_{\tpp},
  \end{align*}
  and again via the construction of the Haar basis,
  each row of $\bH_{n}$ (\ie{}, each column of $\bH_{n}^{\top}$) has only $\log T+1$ non-zero
  entries, we can compute each $d\times 1$ block of $\tilde\btheta_{\tpp}$ using
  $\scO(d\log T )$ computation.
  Finally, observe that in order
  to implement the direction learner,
  we need only compute the $t^{\text{th}}$ $d\times 1$ block
  of $\tilde\btheta_{t}$. Indeed,
  since for each $t$, the vector $\Gt=\et\otimes\gt$ has only $d$ non-zero
  indices, it suffices to retrieve the corresponding indices of $\tilde \bv_{t}$ to
  implement direction learner.
\end{proof}
We note that the memory overhead of maintaining each of these variables can also
likely be reduced by more careful bookkeeping, and
acknowledging the fact that the algorithm only really needs to
retrieve the $t^{\text{th}}$ block of $\Wt$, since the losses are $\Gt=\et\otimes\gt$.
We omit these considerations here for brevity.

    %auto-ignore
\section{Proofs for Section~\ref{sec:coupling} (\SecCoupling)}%
\label{app:coupling}

\subsection{Proof of Theorem~\ref{thm:sequence-reward-regret}}%
\label{app:sequence-reward-regret}
\SequenceRewardRegret*
\begin{proof}
  Thanks to \Cref{prop:dynamic-to-static}, the proof is essentially the same
  as the usual one. We provide the argument here for completeness.

  From \Cref{prop:dynamic-to-static}, $R_{T}(\vec{\cmp})=R_{T}^{\Seq}(\Cmp)=\sumtT \inner{\Gt,\Wt-\Cmp}$
  for $\Gt=\et\otimes\gt$ and $\Cmp=\sumtT \et\otimes\cmp_{t}$. Hence, recalling
  the definition of the Fenchel conjugate, we have
  \begin{align*}
    R_{T}(\vec{\cmp})
    &=
      \sumtT \inner{\gt,\wt-\cmp_{t}}
      =
      \sumtT\inner{\Gt,\Wt-\Cmp}
    =
      -\Wealth_{T}-\sumtT\inner{\Gt,\Cmp}\\
    &\le
      \Big\langle-\sumtT \Gt,\Cmp\Big\rangle-f_{T}^{*}\Big(-\sumtT \Gt\Big)
    \le
      \sup_{\btheta}\inner{\btheta,\Cmp}-f_{T}^{*}(\btheta)
    =
      f_{T}(\Cmp)\ .
  \end{align*}
  Similarly, for the other direction, suppose we have
  $R_{T}(\vec{\cmp})=R_{T}^{\Seq}(\Cmp)\le f_{T}(\Cmp)$ for any $\Cmp$.
  Then re-arranging, we have
  $\Wealth_{T}\ge \inner{-\sumtT\Gt, \Cmp}-f_{T}(\Cmp)$, and since this holds for
  any $\Cmp$, we can choose the one that tightens the bound to get
  $\Wealth_{T}\ge \sup_{\Cmp}\inner{-\sumtT \Gt,\Cmp}-f_{T}(\Cmp)=f_{T}^{*}(-\sumtT \Gt)$.
\end{proof}

%auto-ignore

\section{Supporting Lemmas}
\label{app:lemmas}

\begin{restatable}{lemma}{diffOpProperties}\label{lemma:diff-op-properties}
  Let $\diffOp\in\R^{T\times T}$ be the finite-difference operator,
  having entries
  \begin{align*}
    \diffOpScalar_{ij}=\begin{cases}
            1&\text{if }i=j\\
            -1&\text{if }j=i+1\\
            0&\text{otherwise}
          \end{cases}~.
  \end{align*}
  Then,
  \begin{enumerate}
    \item The inverse of $\diffOp$ the upper-triangular matrix of $1$'s:
          \begin{align*}
            \diffOpScalar_{ij}^{\inv}=\begin{cases}
                              1&\text{if }j\ge i\\
                              0&\text{otherwise}
                      \end{cases},\quad\forall i,j~.
          \end{align*}
    \item The eigenvalues of $\diffOp$ and $\diffOp^{\inv}$ are  $\lambda_{i}=1$ for all $i\in[T]$.
    \item $x\mapsto x^{\top}\diffOp x$ is positive definite.
  \end{enumerate}
  Moreover, the analogous properties hold for the block matrix
  $\diffOp\otimes \bI_d\in\R^{dT\times dT}$.
\end{restatable}
\begin{proof}
  The inverse of $\diffOp$ is
  the upper-triangular matrix $\bDelta$
  characterized by entries
  \begin{align*}
    \Delta_{ij}=\begin{cases}
            1&\text{if }j\ge i\\
            0&\text{otherwise}
            \end{cases}.
  \end{align*}
  To see why, observe that we have $\diffOpScalar_{T,T}\Delta_{T,T}=1$ and for $i<T$ we have
  \begin{align*}
    [\diffOp \bDelta]_{ij}=\sum_{i,j}\diffOpScalar_{ik}\Delta_{kj}= \Delta_{ij}-\Delta_{i+1,j} = \begin{cases}1&\text{if }i=j\\0&\text{otherwise}\end{cases},
  \end{align*}
  and likewise for $[\bDelta\diffOp]_{ij}$. Hence $\diffOp \bDelta=\bDelta\diffOp=I$ and $\bM^{\inv}=\bDelta$.

  Next, since $\diffOp$ and $\diffOp^{\inv}$ are upper-triangular, their eigenvalues are equal to
  their diagonal entries, and hence both have eigenvalues $\lambda_{i}=1$ for
  all $i$.

  To see that the asymmetric matrix $\diffOp$ is positive definite, it suffices to
  show that the symmetric part of $\diffOp$, \textit{i.e.}, the matrix $\diffOp_{S}=(\diffOp+\diffOp^{\top})/2$, is positive definite
  \citep{johnson1970positive}. Luckily, $\diffOp_{S}$ is
  also a well-known variation of the discrete difference
  operator and is known to be positive definite~
  \citep[see, \textit{e.g.}, Theorem 7.4.7 in][]{stoer1980introduction}.

  For the block matrix $\bB=\diffOp\otimes \bI_d$, the inverse is
  given immediately by the inverse property of the Kronecker product:
  $\bB^{\inv}=\brac{\diffOp\otimes \bI_{d}}^{\inv}=\diffOp^{\inv}\otimes \bI_{d}$. We also have that
  $\bB=\diffOp\otimes \bI_{d}$ and $\bB^{\inv}$ have
  eigenvalues $\lambda_{i}=1$ for all $i\in[dT]$, since both are again
  upper-triangular with $1$'s on their main diagonal.
  Finally, we have positive definiteness of $\bB$
  using the fact that the symmetric part of $\bB=\diffOp\otimes \bI_{d}$ is
  $\half(\bB+\bB^{\top})=\half(\diffOp\otimes \bI_d+\diffOp^{\top}\otimes \bI_{d})=\half(\diffOp+\diffOp^{\top})\otimes \bI_{d}$
  by the distributive property, hence $\bB$ is the Kronecker product
  of two symmetric positive definite matrices, so
  $\bB$ is positive definite \citep[Chapter 2]{steeb1997matrix}.
\end{proof}

We borrow the following eigenvalue bound from \cite{wolkowicz1980bounds}.
\begin{restatable}{theorem}{EigenBound}\label{thm:eigen-bound}
  (\citet[Theorem 2.1]{wolkowicz1980bounds}) Let $\bA$ be a symmetric $n\times n$
  matrix with eigenvalues $\lambda_{1}(\bA)\le\ldots\le\lambda_{n}(\bA)$. Then
  \begin{align*}
    \lambda_{\max}(\bA)\le \frac{\Tr(\bA)}{n}+\sqrt{(n-1)\sbrac{\frac{\Tr(\bA^{\top}\bA)}{n}-\brac{\frac{\Tr(\bA)}{n}}^{2}}}~.
  \end{align*}
\end{restatable}

%auto-ignore
\newpage
\section*{NeurIPS Paper Checklist}

\begin{enumerate}

\item {\bf Claims}
    \item[] Question: Do the main claims made in the abstract and introduction accurately reflect the paper's contributions and scope?
    \item[] Answer: \answerYes{} %
    \item[] Justification: Our abstract and introduction clear state the main claims of our paper.

\item {\bf Limitations}
    \item[] Question: Does the paper discuss the limitations of the work performed by the authors?
    \item[] Answer: \answerYes{} %
    \item[] Justification: We clearly state the problem setting and the assumptions therein. We clearly state which results are optimal up to logarithmic factors. We discuss the computational complexity of our newly proposed algorithm at the end of \Cref{sec:haar}. Attaining the tightest matching logarithmic dependencies in our lower bound is still an open problem, as mentioned in the conclusion.

\item {\bf Theory Assumptions and Proofs}
    \item[] Question: For each theoretical result, does the paper provide the full set of assumptions and a complete (and correct) proof?
    \item[] Answer: \answerYes{} %
    \item[] Justification: All the theorems provide the full set of assumptions. All of our main results are proven explicitly either in the main text or in the appendix.

    \item {\bf Experimental Result Reproducibility}
    \item[] Question: Does the paper fully disclose all the information needed to reproduce the main experimental results of the paper to the extent that it affects the main claims and/or conclusions of the paper (regardless of whether the code and data are provided or not)?
    \item[] Answer: \answerNA{} %
    \item[] Justification: This is a theoretical paper that studies theoretical guarantees for online learning algorithms.

\item {\bf Open access to data and code}
    \item[] Question: Does the paper provide open access to the data and code, with sufficient instructions to faithfully reproduce the main experimental results, as described in supplemental material?
    \item[] Answer: \answerNA{} %
    \item[] Justification: This is a theoretical paper that studies theoretical guarantees for online learning algorithms.

\item {\bf Experimental Setting/Details}
    \item[] Question: Does the paper specify all the training and test details (e.g., data splits, hyperparameters, how they were chosen, type of optimizer, etc.) necessary to understand the results?
    \item[] Answer: \answerNA{} %
    \item[] Justification: This is a theoretical paper that studies theoretical guarantees for online learning algorithms.

\item {\bf Experiment Statistical Significance}
    \item[] Question: Does the paper report error bars suitably and correctly defined or other appropriate information about the statistical significance of the experiments?
    \item[] Answer: \answerNA{} %
    \item[] Justification: This is a theoretical paper that studies theoretical guarantees for online learning algorithms.

\item {\bf Experiments Compute Resources}
    \item[] Question: For each experiment, does the paper provide sufficient information on the computer resources (type of compute workers, memory, time of execution) needed to reproduce the experiments?
    \item[] Answer: \answerNA{} %
    \item[] Justification: This is a theoretical paper that studies theoretical guarantees for online learning algorithms.
    
\item {\bf Code Of Ethics}
    \item[] Question: Does the research conducted in the paper conform, in every respect, with the NeurIPS Code of Ethics \url{https://neurips.cc/public/EthicsGuidelines}?
    \item[] Answer: \answerYes{} %
    \item[] Justification: Given the theoretical nature of this work, there are no ethical concerns to be addressed.

\item {\bf Broader Impacts}
    \item[] Question: Does the paper discuss both potential positive societal impacts and negative societal impacts of the work performed?
    \item[] Answer: \answerNA{} %
    \item[] Justification: This is a theoretical paper that studies theoretical guarantees for online learning algorithms. As stated in the guidelines, this is foundational research and it is not tied to particular applications, let alone deployments.
    
\item {\bf Safeguards}
    \item[] Question: Does the paper describe safeguards that have been put in place for responsible release of data or models that have a high risk for misuse (e.g., pretrained language models, image generators, or scraped datasets)?
    \item[] Answer: \answerNA{} %
    \item[] Justification: No models nor data is associated to this paper.

\item {\bf Licenses for existing assets}
    \item[] Question: Are the creators or original owners of assets (e.g., code, data, models), used in the paper, properly credited and are the license and terms of use explicitly mentioned and properly respected?
    \item[] Answer: \answerNA{} %
    \item[] Justification: No assets were used in this paper.

\item {\bf New Assets}
    \item[] Question: Are new assets introduced in the paper well documented and is the documentation provided alongside the assets?
    \item[] Answer: \answerNA{} %
    \item[] Justification: This paper does not release new assets.

\item {\bf Crowdsourcing and Research with Human Subjects}
    \item[] Question: For crowdsourcing experiments and research with human subjects, does the paper include the full text of instructions given to participants and screenshots, if applicable, as well as details about compensation (if any)? 
    \item[] Answer: \answerNA{} %
    \item[] Justification: This paper does not involve crowdsourcing nor research with human subjects.

\item {\bf Institutional Review Board (IRB) Approvals or Equivalent for Research with Human Subjects}
    \item[] Question: Does the paper describe potential risks incurred by study participants, whether such risks were disclosed to the subjects, and whether Institutional Review Board (IRB) approvals (or an equivalent approval/review based on the requirements of your country or institution) were obtained?
    \item[] Answer: \answerNA{} %
    \item[] Justification: This paper does not involve crowdsourcing nor research with human subjects.

\end{enumerate}

\end{document}